\documentclass[twoside,11pt]{article}

\usepackage{blindtext}
\usepackage{appendix}
\usepackage{jmlr2e}

\usepackage{lastpage}

\firstpageno{1}

\usepackage[utf8]{inputenc} 
\usepackage[T1]{fontenc}    
\usepackage{hyperref}       

\usepackage{subfigure,tikz,verbatim, pifont}
\usepackage{algorithm,algorithmic} 

\usepackage{url}            
\usepackage{booktabs}       
\usepackage{nicefrac}       
\usepackage{bm,dsfont}
\usepackage{xcolor}
\usepackage{amssymb}

\usepackage{amsfonts}       
\usepackage{mathrsfs}
\usepackage{float}
\usepackage{enumitem}
\usepackage{multirow} 
\usepackage{caption} 
\usepackage{setspace}
\usepackage[intlimits]{amsmath}
\usepackage{graphicx}
\usepackage{grffile}

\usepackage{xcolor}

\hypersetup{
    colorlinks=true,
    linkcolor=red,
    anchorcolor=green,
    citecolor=blue
    }
\numberwithin{equation}{section}

\def\ba{{\bm a}}

\def\bu{{\bm u}}

\def\bw{{\bm w}}
\def\bx{{\bm x}}
\def\by{{\bm y}}

\def\bK{{\bm K}}

\def\ga{{\bm \alpha}}
\def\gb{{\bm \beta}}

\def\gx{{\bm \xi}}

\def\Max{\rm maximize}

\def\R{{\mathbb R}}

\def\Fcal{{\mathcal F}}

\def\Tcal{{\mathcal T}}
\def\SVcal{{\mathcal{SV}}}
\def\Min{{\mathrm{minimize}}}
\def\ST{{\mathrm{subject\  to}}}

\usepackage{cleveref}

\begin{document}


\title{Towards Convexity in Anomaly Detection: A New Formulation of SSLM with Unique Optimal Solutions}



\author{
       \name Hongying Liu  \email liuhongying@buaa.edu.cn\\
       \addr School of Mathematical Sciences,\\
       Beihang University\\
       Beijing, 100191, China
       \AND
\name Hao Wang* \email haw309@gmail.com \\
       \addr School of Information Science and Technology\\
       ShanghaiTech University\\
       Shanghai, 201210, China
       \AND
\name Haoran Chu \email chuhr2023@shanghaitech.edu.cn\\
       \addr School of Information Science and Technology\\
       ShanghaiTech University\\
       Shanghai, 201210, China
        \AND
\name Yibo Wu \email wybmath@buaa.edu.cn\\
       \addr School of Mathematical Sciences,\\
       Beihang University\\
       Beijing, 100191, China
       }

 \editor{ }

\maketitle

\begin{abstract}

An unsolved issue in widely used methods such as Support Vector Data Description (SVDD) and Small Sphere and Large Margin SVM (SSLM) for anomaly detection is their nonconvexity, which hampers the analysis of optimal solutions in a manner similar to SVMs and limits their applicability in large-scale scenarios. In this paper, we introduce a novel convex SSLM formulation which has been demonstrated to revert to a convex quadratic programming problem for hyperparameter values of interest. Leveraging the convexity of our method, we derive numerous results that are unattainable with traditional nonconvex approaches. We conduct a thorough analysis of how hyperparameters influence the optimal solution, pointing out scenarios where optimal solutions can be trivially found and identifying instances of ill-posedness.
Most notably, we establish connections between our method and traditional approaches, providing a clear determination of when the optimal solution is unique---a task unachievable with traditional nonconvex methods. We also derive the $\nu$-property to elucidate the interactions between hyperparameters and the fractions of support vectors and margin errors in both positive and negative classes.
\end{abstract}

\begin{keywords}
support vector data description, small sphere and large margin, anomaly detection, $\nu$-property, convex optimization
\end{keywords}

\section{Introduction and Background}

Anomaly detection, also known as outlier detection, is a crucial task in data analysis and aims to recognize whether a new observation belongs to the same distribution as existing observations. It is applicable across a wide range of domains where detecting anomalous data patterns is crucial, including anomaly detection, quality control, fraud detection, image analysis, and remote sensing. In these applications, only large datasets of the `normal' or `positive' condition (also known as positive examples, labeled as $y_i=+ 1$) are encountered, and the goal is to classify data into this primary class while not explicitly modeling or considering other classes of `abnormalities' (also known as negative examples, labeled as
$y_i=-1$).


A wide array of anomaly detection methods exists, and the support vector machine (SVM) and the neural networks are the most popular tools for this purpose \citep{JMLRsteinwart05a,zhou2024}. Several ``(hyper)sphere-based'' models have emerged among the various SVM-based approaches, each proving efficient and powerful in anomaly detection. These include the ball volume minimization model (\citeauthor{Scholkopf95}, \citeyear{Scholkopf95}), the support vector data description (SVDD) (\citeauthor{Tax04} , \citeyear{Tax04}), the soft-boundary SVDD (\citeauthor{Tax99}, \citeyear{Tax99}), as well as its convexified variant (\citeauthor{Chang14}, \citeyear{Chang14}), and the small sphere and large-margin (SSLM) approach (\citeauthor{wu2009small}, \citeyear{wu2009small}).
At the core of these methods lies the construction of a hypersphere that encapsulates normal instances while ensuring that most, if not all, of the normal data points lie within or on the boundary of the hypersphere. This often leads to quadratically constrained quadratic programming (QCQP) problems involving the radius and center point of the hypersphere. These methods, particularly SVDD and SSLM, offer numerous advantages. They exhibit robustness when dealing with high-dimensional data, can handle nonlinear relationships through kernelization, and demonstrate a strong ability to generalize well to unseen data. Furthermore, they are valued for their interpretability, as anomalies are identified based on their distance from the center of the hypersphere.

Despite their wide application and numerous advantages, these methods also come with limitations, one of which stems from their nonconvex formulation, which is the central focus of this paper. The model proposed by \cite{Scholkopf95} assume that normal data are encapsulated within a single hypersphere, and the optimal hypersphere is simply the one with the smallest volume, resulting in a convex model. The optimal hypersphere \emph{global} can be easily determined in this case. However, real applications rarely accommodate this assumption, and such a model may become highly sensitive when dealing with outliers, non-convex anomalies, or data with complex structures. Modern models such as SVDD and SSLM allow negative examples and the misclassification of both positive and negative examples. Although the optimal solutions of such models possess powerful interpretation and generalization abilities, they give rise to \emph{nonconvex} QCQP problems.  This is a well-known critical issue that remains unsolved, with many associated issues and limitations stemming from it.

\subsection{Related work}
Many researchers have acknowledged the prevailing challenge posed by the nonconvexity issue inherent in sphere-based Support Vector Machines (SVMs). Over the past decades, several efforts have been devoted to addressing this intricate problem. Their focus has been centered on developing innovative strategies and algorithms to overcome the limitations imposed by nonconvexity, in a bid to enhance the performance and efficiency of sphere-based SVMs.

Consider the training data
$
\bx_1,\cdots, \bx_{\ell}\in \mathcal{X},
$
where for each $i\in\{1, \cdots, \ell\}$  (in compact notation: $i\in [1,\ell]$)  the vector $\bx_i$ represents the  {input vector, $\ell\in \mathbb{N}$ is the number of observations and $\mathcal{X} $ is the data domain and
is assumed to be a subset of $\R^N$}. 
Denote $\ba$ and $R$ for the ball center and radius, respectively. 
\cite{Scholkopf95} minimized the volume $R^2$ of the ball while requiring all examples to be contained in this ball $\|\bx_i-\ba\|^2\leq R^2, i\in [1,\ell]$.
A test point is classified as normal if it is enclosed by this hypershere and abnormal otherwise.

However, this model is very sensitive to outlier data points. If there are one or several outlier data points in the training set, a very large hypersphere will be obtained. To accommodate the existence of outliers in the training data, \cite{Tax1999DataDD,Tax99} proposed soft-boundary-SVDD:
\begin{equation}\label{prob.softsvdd}
\begin{array}{cl}
\mathop{\Min}\limits_{\ba\in\R^N, R\in\R,\gx\in\R^\ell} &R^2+C\sum_{i=1}^\ell\xi_i \\
\ST &\|\bx_i-\ba\|^2\leq R^2+\xi_i,\xi_i\ge0, i\in [1,\ell],
\end{array}
\end{equation}
where  $\xi_i$ is the relaxed variable and the positive parameter $C$ controls the trade-off between the volume and the empirical errors.

The nonconvexity term $R^2$ in the constraints of these methods 
can be removed by introducing a new nonnegative variable $r = R^2$. \cite{Chang14} pointed out
that the constraint $r\geq0$ in this problem is not necessary for $C>1/\ell$, meaning in this case  the optimal solution to the relaxation problem
\begin{equation}\label{eq:SVDD_normr}
\begin{array}{ll}
\mathop{\Min}\limits_{\ba\in \R^N, r\in\R,\gx\in\R^\ell} &r+C\sum_{i=1}^\ell\xi_i\\
\ST &\|\bx_i-\ba\|^2\leq r+\xi_i, \xi_i\ge0, i\in [1,\ell]
\end{array}
\end{equation}
always has $r\geq0$. If $C\leq 1/\ell$, then
$\ba_*=\frac1\ell\sum_{i=1}^\ell\bx_i$ and $r_*=0$ is an optimal solution for \eqref{eq:SVDD_normr}.  Notice that \eqref{eq:SVDD_normr} is a convex problem and satisfies the Slater condition, ensuring strong duality.
 The dual problem of \eqref{eq:SVDD_normr} is
\begin{equation}\label{eq:dSVDD_normr}
\begin{array}{ll}
\mathop{\Max}\limits_{\ga\in\R^\ell} &\sum_{i=1}^\ell\alpha_i\bx_i^T\bx_i-\sum_{i=1}^\ell\sum_{j=1}^\ell\alpha_i\alpha_j\bx_i^T\bx_j\\
\ST &\sum_{i=1}^\ell\alpha_i=1,\\
&0\leq \alpha_i\leq C,i\in [1,\ell].
\end{array}
\end{equation}
\citeauthor{Wang11} (\citeyear[Theorem 2]{Wang11}) and
\citeauthor{Chang14} (\citeyear[Theorem 2]{Chang14}) showed that the optimal sphere center is always unique for the given dataset.   As for the optimal ball radius,
\citeauthor{Wang11} (\citeyear[Theorem 3]{Wang11}) claimed that the sphere radius $r_*$ is not the unique optimal radius {\em if and only if}
\begin{equation}\label{error_unique_condition}
|\Fcal_{r_*}|C=1
\end{equation}
where $\Fcal_{r_*}=\{i:y_i=1, \|\bx_i-\ba_*\|^2-r_*>0\}$ is the examples that lie outside the sphere. Furthermore,
\citeauthor{Chang14} (\citeyear[Theorem 5]{Chang14}) showed that $r$ is optimal for \eqref{eq:SVDD_normr} if and only if
\begin{equation}\label{eq:alpharange}
\max_{i:\alpha_i^*<C}\|\bx_i-\ba_*\|^2\leq r\leq\min_{i:\alpha_i^*>0}\|\bx_i-\ba_*\|^2.
\end{equation}
Suppose $\ga^*$ is an optimal solution for the dual problem \eqref{eq:dSVDD_normr}. 
\citeauthor{Wang11} (\citeyear[Theorem 4]{Wang11}) claimed that the radius of the ball is unique {\em if and only if} there exists $i_0\in\{1,\cdots,\ell\}$ such that $\alpha_{i_0}^*\in(0,C)$.

 { Let $k:\mathcal{X}\times\mathcal{X}\to \R$ be a positive definite kernel with reproducing kernel Hilbert space (RKHS) $H$ and $\Phi$ be a feature map $\mathcal X\to H$. Then the inner product in $H$ can be computed by evaluating with
$k(\bx,\by)=\langle \Phi(\bx), \Phi(\by) \rangle$ and denote $k_{ij}=k(\bx_i,\bx_j)$.}
To separate the data set from the origin, \cite{Scholkopf99,Scholkopf01} proposed one-class support vector machine (OCSVM):
\begin{equation}\label{prom.ocsvm} 
\begin{array}{cl}
\mathop{\Min}\limits_{\bw\in H,\gx\in\R^\ell,\rho\in\R}&\frac12\|\bw\|^2-\rho+\frac{1}{\nu \ell}\sum_{i=1}^\ell\xi_i\\
\ST &\langle\bw,\Phi(\bx_i)\rangle\geq\rho-\xi_i,\xi_i\geq0, i\in [1,\ell],
\end{array}
\end{equation}
where $\nu\in(0,1]$ is a given parameter. Note that although \eqref{eq:SVDD_normr} is an equivalent convex formulation of \eqref{prob.softsvdd}, it remains the quadratic constraints. 

If negative examples (objects which should be rejected) are available, they can be incorporated into the training set to improve the description.  For convenience, we assume that target objects are labeled $y_i=1$ for $i\in[1,m]$ and outlier objects are labeled $y_i=-1$ for $i\in[m+1,\ell]$.  \cite{Tax01,Tax04} extended the SVDD to  learn a tight hyperball that encloses the majority of the normal data points while excluding the anomalies:
\begin{equation}\label{prob.negativesvdd} 
\begin{array}{cl}
\mathop{\Min}\limits_{\ba\in H,R\in\R,\gx\in\R^\ell} &  R^2+C\sum_{i=1}^{m}\xi_i+D\sum_{i=m+1}^\ell\xi_i\\
\ST
& \|\Phi(\bx_i)-\ba\|^2\leq R^2+\xi_i,\xi_i\ge0, i\in [1,m]\\
& \|\Phi(\bx_i)-\ba\|^2\geq R^2 -\xi_i,\xi_i\ge0,i\in [m+1,\ell],
\end{array}
\end{equation}
where positive parameters $C$ and $D$ together control the trade-off between the volume and the empirical errors. These constraints enforce that the negative examples are located outside the hyperball, effectively pushing the boundary away from the anomalies.
The addition of negative examples allows the model to better capture the boundary between normal and abnormal instances, potentially leading to improved performance in scenarios where anomalies are more diverse or prevalent. 

\cite{wu2009small} proposed  SSLM  approach, which aims to define a hyperball that tightly encloses the majority of normal data points while maintaining a large margin from the decision boundary to separate normal data from outliers. The idea is to create a compact representation of positive examples while maximizing the separation between the positive  and negative instances  to achieve robust anomaly detection performance:
 \begin{equation}\label{prob.sslm}
\begin{array}{cl}
\mathop{\Min}\limits_{\ba\in H, R,\rho\in\R,\gx\in\R^\ell} & R^2-\bar\nu\rho^2+\frac{1}{\bar\nu_1 m}\sum_{i=1}^m\xi_i+\frac{1}{\bar\nu_2 n}\sum_{i=m+1}^\ell\xi_i\\
\ST & \|\Phi(\bx_i)-\ba\|^2\leq R^2+\xi_i,\xi_i\ge0, i\in [1,m]\\
& \|\Phi(\bx_i)-\ba\|^2\geq R^2+\rho^2-\xi_i,\xi_i\ge0,i\in [m+1,\ell],
\end{array}
\end{equation}
where $n=\ell-m$, and $\rho^2\geq0$ is the margin between the surface of the hyperball and the outlier training data,  $\gx\in\R^\ell$ is the vector of slack variables, and the positive parameter $\bar\nu$  controls the trade-off between the volume and margin, positive $\bar\nu_1$ and $\bar\nu_2$ control the trade-off between the volume and the errors. \eqref{prob.sslm} can result in a closed and tight boundary around the normal data. \cite{wu2009small} given the dual problem of \eqref{prob.sslm}:
\begin{equation}\label{prob.sslmdual1}
	\begin{array}{ll}
		\mathop{\Max}\limits_{\ga\in\R^\ell} &   -\sum\limits_{i=1}^{\ell}\sum\limits_{j=1}^\ell y_iy_{j}k_{ij}\alpha_i\alpha_j+\sum\limits_{i=1}^\ell y_ik_{ii}\alpha_i\\
		\ST & \sum\limits_{i=1}^\ell y_i\alpha_i=1,\quad \sum\limits_{i=1}^\ell\alpha_i = 2\bar\nu+1,\\
	    	& 0\leq\alpha_i\leq \tfrac{1}{\bar\nu_1m}, i\in [1,m];\quad  0\leq\alpha_i\leq \frac{1}{\bar\nu_2n},i\in [m+1,\ell].
	\end{array}
\end{equation}

The term $R^2$ in the constraints of \eqref{prob.negativesvdd} and \eqref{prob.sslm} associated with the positive examples can be removed by introducing new nonnegative variable $r = R^2$. Similarly, the   term $\rho^2$ in the objective of \eqref{prob.sslm} can be removed by introducing new nonnegative variable $t = \rho^2$. While the nonconvexity term $\|\ba\|^2$ appearing in the constraint with negative examples in
\eqref{prob.negativesvdd} and \eqref{prob.sslm} is inherent, which results in a nonconvex QCQP problems. It is worthwhile to ask whether a convex quadratic program (QP) formulation for \eqref{prom.ocsvm}, \eqref{prob.negativesvdd} and \eqref{prob.sslm}, as the well-known support vector machines (SVM)
\begin{equation}\label{eq:svm}
\begin{array}{cl}
\mathop{\Min}\limits_{\bw\in H,\gx\in\R^\ell,\rho\in\R}&\frac{\nu}{2}\|\bw\|^2+\frac{1}{\ell}\sum_{i=1}^m\xi_i+\frac{b}{\ell}\sum_{i=m+1}^{\ell}\xi_i\\
\ST &y_i(\langle\bw,\Phi(\bx_i)\rangle+b)\geq 1-\xi_i,\xi_i\geq0, i\in [1,\ell].
\end{array}
\end{equation}
  
 {It is worth noting that the norm in \eqref{prom.ocsvm}, \eqref{eq:svm}, \eqref{prob.negativesvdd} and \eqref{prob.sslm} are induced by the inner product in RKHS $H$ for $\bw, \ba\in H$. It reduced to the Euclidean norm if $\Phi(\bx)=\bx$.}

\subsection{Motivation}

Nonconvexity is a prevalent issue in hypersphere-based SVM models, particularly in well-known methods such as  SVDD  and  SSLM. This poses challenges in model property analysis, hyperparameter selection, and algorithm design.

\emph{Model Property Analysis}.
The nonconvexity of hypersphere-based SVMs significantly complicates the analysis of their solutions, especially when compared to hyperplane-based SVMs \citep{Vapnik1995NatureofSLT,2007Learning,Steinwart2008SVM}. Fundamental questions, such as whether a unique solution exists, remain unresolved. A key issue is the lack of strong duality, which is essential for working effectively in kernel space. Nonconvexity can result in an incorrect Lagrangian dual problem, maintaining a positive duality gap between the primal and dual solutions. Moreover, solving nonconvex problems to global optimality is often computationally difficult, if not impossible, and the solutions obtained by solvers are typically only local optima, limiting their usefulness for analysis.

\emph{Hyperparameter Selection}.  One of the most pressing issues with nonconvex models is the challenge of hyperparameter tuning. Parameters such as the sphere size, margin, and kernel function directly impact performance and generalization ability. The difficulty arises from the inability to guarantee ``global'' minima for each hyperparameter, as local minima make it hard to assess whether a chosen value is truly optimal. Additionally, balancing the size of the hypersphere (which represents normal data) against the margin separating outliers can be problematic, as local minima increase the sensitivity to outliers, resulting in suboptimal anomaly detection performance.

\emph{Algorithm Design}.
Algorithm design for hypersphere-based SVMs has been underexplored, particularly when compared to the more established hyperplane-based SVMs \eqref{eq:svm}. For instance, \cite{Chapelle07} demonstrated that the primal problem of SVMs can be efficiently solved, and \cite{Pegasos11} proposed the Pegasos algorithm, which uses stochastic subgradient descent to solve the primal problem of SVMs. We believe that resolving the nonconvexity issue in hypersphere-based SVMs would pave the way for the development of similarly powerful algorithms.

\emph{Approximation Methods and Efficiency}.
Numerous studies have proposed approximation methods to address the nonconvexity of SVDD and SSLM. For example, \cite{chen2015robust} approximated SVDD using a Gaussian radial basis function kernel, facilitating easier solution acquisition. \cite{wang2023non} replaced the hinge loss in SSLM with a ramp loss, converting the nonconvex problem into a difference-of-convex  form, which was then solved using the concave-convex procedure. \cite{guo2020incremental}  incorporated incremental learning into SSLM to improve training efficiency.

 {\emph{Bilevel Optimization and the Challenge of Nonconvexity}.
Bilevel optimization, which has recently attracted significant attention as a powerful tool for hyperparameter selection and meta-learning, provides a natural framework where the upper-level problem optimizes hyperparameters while the lower-level problem corresponds to model training. However, nonconvexity—whether in the upper- or lower-level—poses severe challenges. A nonconvex upper-level objective admits multiple local minima, complicating the search for globally optimal hyperparameters. More critically, nonconvexity in the lower-level problem undermines the very foundation of bilevel methods: the solution set of the inner problem can be non-unique, discontinuous, or non-differentiable, rendering implicit differentiation and hypergradient techniques inapplicable \citep{zhang2024introduction}. In such cases, the overall bilevel formulation is not only highly nonconvex but often computationally intractable (NP-hard), extremely sensitive to initialization, and prone to instability under standard gradient-based algorithms \citep{liu2022bome}. These issues highlight that ensuring convexity—at least at the lower level—is essential for making bilevel optimization practically reliable and theoretically well-grounded. This challenge directly motivates our convex reformulation, which guarantees tractability and stability while preserving the modeling power of the bilevel framework.}


Nonconvexity also hampers many techniques designed to enhance computational efficiency during training. SSLM, in particular, involves multiple variables that cannot be explicitly represented as a linear combination of training samples. As a result, researchers \citep{cao2020multi, WANG2022382, YUAN2021107860} have focused on developing screening rules to identify and eliminate non-support vectors, reducing the size of the nonconvex problem. The first such rule, proposed by \cite{cao2020multi}, integrates the $\nu$-property, KKT conditions, and variational inequalities specific to SSLM.

This ongoing research demonstrates that addressing the nonconvexity in hypersphere-based SVMs can lead to significant advances in both theoretical analysis and practical algorithm development.

\subsection{Contribution}

The motivations discussed above prompted us to propose a convex optimization formulation to obtain the \emph{global} optimal solution for the SSLM with given hyperparameters. This allows us to analyze our method in a manner similar to SVMs, addressing a longstanding unresolved issue in hypersphere-based SVM models like SSLM and SVDD.

\begin{itemize}
\item \emph{Convex Reformulation}: We introduce a new kernelized SSLM model that is convex for any set of hyperparameters. Specifically, we prove that for prescribed values of these hyperparameters, our model reduces to either a convex quadratic programming (QP) or linear programming (LP) problem. This ensures that the \emph{global} optimal solution can be found efficiently using existing large-scale QP/LP solvers.

\item \emph{Hyperparameter and Solution Interaction}: We analyze the relationship between hyperparameters and the resulting optimal solutions, identifying ``inappropriate''  hyperparameter values that lead to ill-posed problems (where the optimal margin and objective value tend to infinity). Additionally, we highlight ``trivial'' hyperparameter choices for which the optimal solution can be explicitly determined. This analysis provides valuable guidance for hyperparameter selection.

\item \emph{Uniqueness and Dual Characterization}: We demonstrate the uniqueness of the optimal solution for relevant hyperparameter values and characterize these solutions using dual variables. This type of analysis is rarely possible in nonconvex models, highlighting a key advantage of our convex approach.

\item \emph{$\nu$-Property}: We derive the $\nu$-property of our model, describing the relationship between hyperparameters and the fractions of support vectors and margin errors in both the positive and negative classes. This property allows for flexible control of the SVM model's complexity, which is particularly useful in situations such as class imbalance (common in anomaly detection) or when direct control over the number of support vectors is desired.

\item \emph{Connections to Traditional Methods}: Our method establishes clear connections with traditional approaches like SVDD and SSLM, making it possible to apply the analyses and results presented here to those models. This is a significant advancement towards overcoming the limitations posed by nonconvexity in traditional hypersphere-based SVMs.

\item \emph{Numerical Performance}: Through numerical experiments, we demonstrate that our convex method consistently outperforms traditional SSLM. While the nonconvexity of SSLM can lead to convergence to non-global minima when using local optimization algorithms, our convex approach reliably converges to the global minima, ensuring both efficiency and accuracy.
\end{itemize}

\subsection{Outline}

In Section \ref{sec.model}, we introduce the convex   SSLM model, and analyze the influence of hyperparameters on the optimal model.
In Section \ref{sec.connect}, the connection of our proposed model to traditional SVDD and SSLM models is established.  Section \ref{sec:unique_of_solution} presents the derivation of the Lagrangian dual problem for the proposed model, proves the uniqueness of optimal solutions for relevant hyperparameter values, and characterizes the optimal solution using dual variables, highlighting the $\nu$-property of our convex QP formulation. In Section \ref{sec.test}, we provide experimental results that demonstrate the advantages of our method over traditional SSLM. Finally, Section \ref{sec.conclude} concludes the paper.


\section{The Convex SSLM}\label{sec.model}

In this section, we present our \emph{convex}  SSLM (CSSLM) on the data set
$\{ (\bx_i, y_i)\}_{i=1}^\ell$ with $y_i=1, i \in [1,m]$ and $y_i=-1, i\in [m+1,\ell]$  with $\ell:=m+n$.
The $m$ positive examples are required to be contained in the small hypersphere with radius $r \ge 0$,
\[
\|\Phi(\bx_i)  - \ba\|^2 \le r,\quad y_i=1.
\]
Let $t \ge 0$ be the margin between the two spheres.  The $n$ negative examples are excluded from the hypersphere
\[
  \|\Phi(\bx_i)  - \ba\|^2 \ge r + t,\quad y_i=-1.
\]
The empirical risk is $(\|\Phi(\bx_i)-\ba\|^2- r)_{+}$ if $y_i=1$ and $b(r+t-\|\Phi(\bx_i)-\ba\|^2)_{+}$ if $y_i=-1$ with $(z)_+:=\max\{z,0\}$.
By minimizing the small sphere and the empirical risk, we arrive at the regularized empirical risk minimization formulation of \eqref{prob.sslm} 
\begin{equation}\label{prob.unc}  \tag{$\mathcal{F}$}
\inf\limits_{\ba\in H,  r\ge 0 ,t\geq 0}\  g(\ba,r,t)
\end{equation}
with
\[
g(\ba, r, t)=\tfrac{1}{2}(\nu r-\mu t)+\tfrac{1}{2\ell}\sum\limits_{i=1}^m(\|\Phi(\bx_i)-\ba\|^2-r)_+ +\tfrac{b}{2\ell}\sum\limits_{i=m+1}^\ell( r+t-\|\Phi(\bx_i)-\ba\|^2)_+.
\]
 Here prescribed
   $\nu >0 $  controls the trade-off between the volume and the loss,
 $\mu \ge  0$ controls the trade-off between the volume and the margin,
 and $b>0$ controls the trade-off between the loss by positive examples and
 that by negative examples.

Before proceeding to propose  the CSSLM based on \eqref{prob.unc},  we first point out in the following (See Appendix \ref{sec:appendA} for the proof) that  the objective  of  \eqref{prob.unc} is unbounded below if  the value of $\mu$ is set too large.

\begin{lemma}\label{prop:SSLMwj} 
 If  $\mu>\min\{m/{\ell}, bn/{\ell}\}$, then     \eqref{prob.unc}
 is unbounded below.  
\end{lemma}

{The parameter $\mu \le \min\{m/{\ell}, bn/{\ell}\}$ is necessary to ensure that the objective remains bounded. The margin term 
$-\mu t/2$ decreases the objective as $t$ grows, and the hinge losses from positive and negative examples act as a counterbalance. Positive examples can offset at most their dataset proportion $m/\ell$, while weighted negatives contribute at most $bn/\ell$. If $\mu$ exceeds the smaller of these, the margin term dominates, causing $t\to\infty$ and the objective to be unbounded.}
 In the rest of this paper,  we assume $\nu$ and $\mu$ are selected such that
\begin{equation}\label{all.value}
\boxed{\mu \le \min\{\tfrac{m}{\ell}, \tfrac{bn}{\ell}\}.}
\end{equation}
Obviously,  \eqref{prob.unc} is a nonconvex model. We do not suggest directly solving an \emph{nonconvex} model like \eqref{prob.unc}.  We show that for all values satisfying \eqref{all.value},
 the \emph{global} optimal solution of \eqref{prob.unc} can be found by solving a \emph{convex} problem.

 First, suppose  { the regularization strength is smaller than the ratio of positive examples in the dataset. Since the average slope of the hinge-loss from positive samples is $m/\ell$, this balance ensures that the loss can counteract the regularization, making the problem a convex quadratic program. This condition is given as 
 \begin{equation}\label{eq:nondegeneratecondition-nu}
\boxed{\nu+\mu<m/{\ell}.}
\end{equation}
}
Then the global optimal solution of \eqref{prob.unc} can be found by solving the problem
 \begin{equation}\label{prob.csslmunc}
\inf\limits_{\ba\in H,  r\in \R, t\geq 0}\  g(\ba,r,t).
\end{equation}
Compared with \eqref{prob.unc}, the absence of constraint $r\ge 0$ represents a significant difference, which enables
\eqref{prob.csslmunc} to be convex in the absence of negative example.

\begin{theorem}\label{pro:non_degenerate}
If \eqref{eq:nondegeneratecondition-nu} holds,
then the global   optimal solution $(\ba_*,  r_*, t_*)$ of \eqref{prob.csslmunc}
satisfies $r_* \ge 0$  and therefore  is global optimal for \eqref{prob.unc}.  Conversely,  every global minimizer
$(\bar\ba,\bar r,\bar t)$ of \eqref{prob.unc}   is also optimal for \eqref{prob.csslmunc}.
\end{theorem}

\begin{proof}
Suppose  $(\ba_*, r_*, t_*)$  is the optimal solution of \eqref{prob.csslmunc}.  
Assume by contradiction that $r_* < 0$. 
To derive a contradiction, we now investigate the value of $g$ in two cases based on the value of $r_*+t_*$.
For this purpose, define
	  \begin{equation}\label{eq:degenerate-spherecenter}
	\hat{\ba} :=     \arg\min_{ \ba\in H} \tfrac{1}{m}\sum_{i=1}^m\|\Phi(x_i)-\ba\|^2 = \tfrac{1}{m}\sum_{i=1}^m\Phi(x_i).
\end{equation}

If $r_*+t_* 
\leq0$, we have
$$
\begin{array}{ll}
g(\ba_*,r_*,t_*)&=\frac{1}{2}(\nu+\mu-\frac{m}{\ell})r_*  -\frac{\mu}{2} (r_*+  t_*) +\frac{1}{2\ell}  \sum_{i=1}^{m}\|\Phi(\bx_i)-\ba_*\|^2 \\
&\geq\frac{1}{2}(\nu+\mu-\frac{m}{\ell})r_*+\frac{1}{2\ell} \sum_{i=1}^{m}\|\Phi(\bx_i)-\ba_*\|^2\\
&> \frac{1}{2\ell} \sum_{i=1}^{m}\|\Phi(\bx_i)-\ba_*\|^2\\
&\geq \frac{1}{2\ell} \sum_{i=1}^{m}\|\Phi(\bx_i)-\hat\ba\|^2\\
&=g(\hat{\ba},0,0),
\end{array}
$$
where the first inequality follows from $r_*+t_*\le 0$ and $\mu\ge  0$, and the strict inequality is by
\eqref{eq:nondegeneratecondition-nu}
and $r_*<0$, and
the last inequality is due to the fact that
$\hat\ba $ is the minimizer of $\sum_{i=1}^{m}\|\Phi(\bx_i)-\ba\|^2$ because \eqref{eq:degenerate-spherecenter}.
Obviously, $(\hat \ba, 0, 0)$ is feasible for \eqref{prob.csslmunc}. This contradicts the optimality of $(\ba_*,r_*,t_*)$.

If $r_*+t_*   >0$, we have
$$
\begin{aligned}
&g(\ba_*,r_*,t_*) \\
 =&  \tfrac{1}{2}(\nu+\mu-\tfrac{m}{\ell})r_* -\tfrac{\mu}{2} (r_*+  t_*) +\tfrac{1}{2\ell}  \sum\limits_{i=1}^{m}\|\Phi(\bx_i)-\ba_*\|^2 +\tfrac{b}{2\ell}\sum\limits_{i=m+1}^\ell (r_*+t_*-\|\Phi(\bx_i)-\ba_*\|^2)_+\\
 >&    -\tfrac{\mu}{2}(r_*+t_*)+\tfrac{1}{2\ell} \sum\limits_{i=1}^{m}\|\Phi(\bx_i)-\ba_*\|^2+\tfrac{b}{2\ell}\sum\limits_{i=m+1}^\ell (r_*+t_*-\|\Phi(\bx_i)-\ba_*\|^2)_+\\
 =&   g(\ba_*,0,r_*+t_*),
\end{aligned}
$$
where the inequality follows by \eqref{eq:nondegeneratecondition-nu} and $r_*<0$.
Notice that $(\ba_*, 0  ,r_*+t_*)$ is also feasible for \eqref{prob.csslmunc}  by $r_*+t_*>0$.
This contradicts the optimality of $(\ba_*,r_*,t_*)$.
Overall, we have shown that $r_*\geq0$, so that $(\ba_*, r_*, t_*)$ is also feasible for
\eqref{prob.unc}.  Hence, it is globally optimal for \eqref{prob.unc}.

Conversely, let $(\bar\ba, \bar r, \bar t)$ be globally optimal for  \eqref{prob.unc}. It follows that
$g(\bar\ba, \bar r, \bar t) \le g(\ba_*, r_*, t_*)$  because  $(\ba_*, r_*, t_*)$ is feasible for  \eqref{prob.unc}; moreover, $(\bar\ba, \bar r, \bar t)$ is also feasible for \eqref{prob.csslmunc}.
Hence, $(\bar\ba, \bar r,\bar t)$ is also optimal for \eqref{prob.csslmunc}.
\end{proof}

{ Now we give the convex QP formulation of \eqref{prob.csslmunc}. We introduce the intermediate variable
\begin{equation}\label{eq:transform-r-s}
s = \|\ba\|^2 - r,
\end{equation}
 instead of using $r$ as the variable, so that the empirical risk becomes $(e_i)_+$ if $y_i=1$ and $b(t-e_i)_+$ if $y_1=-1$ with $e_i=(k_{ii}+s)/2- \langle \ba,\Phi(\bx_i) \rangle$.
Now \eqref{prob.csslmunc}  reverts to
\begin{equation}\label{prob.csslmunc-s}
 \inf\limits_{\ba\in H, s\in\R,  t\geq 0}\   f(\ba,s,t): =\tfrac{\nu}{2}(\|\ba\|^2 - s)-\tfrac{\mu}{2} t+ \tfrac{1}{\ell}\sum\limits_{i=1}^m (e_i)_++\tfrac{b}{\ell}\sum_{i=m+1}^l(\tfrac{t}{2}-e_i)_+.
\end{equation}
}
The radius $r$ of the sphere appears as a latent variable in this model, and the small sphere is attained by minimizing
$\|\ba\|^2 -s$ instead of $r$.
By introducing the auxiliary  variable $\xi_i$,  \eqref{prob.csslmunc-s} is indeed a
\emph{convex QP}  model\footnote{Note that the objective function of \eqref{prob.csslmunc-s} is a scaled one of  \eqref{prob.primal1} with a factor $1/\ell$.}:
\begin{equation}\label{prob.primal1}\tag{$\mathcal{P}$} 
	\begin{array}{cl}
		\mathop{\rm\Min}\limits_{\ba\in H,s,t\in\R, \gx\in\R^\ell} &h(\ba,s,t,\gx):=\frac{\ell\nu}{2}(\|\ba\|^2-s)-\frac{\ell\mu}{2}t+\sum\limits_{i=1}^{m}\xi_i+b\sum\limits_{i=m+1}^\ell\xi_i\\
		\ST &y_i[\frac{k_{ii}}{2}-\langle\Phi(\bx_i), \ba\rangle+\frac{s}{2}]  \hspace{1.6em} -\xi_i\leq0,\    \xi_i\ge0, \quad i\in [1,m]\\
		&y_i[\frac{k_{ii}}{2}-\langle\Phi(\bx_i), \ba\rangle+\frac{s}{2}]+\frac{t}{2}-\xi_i\leq0,\    \xi_i\ge0, \quad i\in [m+1,\ell]\\
		&  \hspace{13.3em}  t\ge 0.
	\end{array}
\end{equation}

 It is often more convenient for SVM methods to train in the kernel space without knowing
the feature mapping $\phi$ explicitly, then transfer \eqref{prob.primal1} from the kernel space to the feature mapping $\Phi$.
Then we have the convex QP problem
{
\begin{equation}\label{prob.primaldual1}\tag{$\mathcal{K}$}
\min\limits_{\gb\in\R^\ell,s\in\R, t\ge 0} \tfrac{\nu}{2}(\gb^T\bK\gb-s)-\tfrac{\mu}{2} t+\tfrac{1}{\ell}\sum\limits_{i=1}^{m}\Big(\tfrac{k_{ii}+s}{2} - \sum_{j=1}^\ell \beta_j k_{ij}\Big)_{+}+\tfrac{b}{\ell}\sum\limits_{i=m+1}^\ell\Big(\tfrac{t}{2}-\Big(\tfrac{k_{ii}+s}{2}- \sum_{j=1}^\ell \beta_j k_{ij}\Big)\Big)_{+}.
\end{equation}
}
Problem \eqref{prob.primaldual1} is suitable to incorporate with different kernel techniques.
The sphere center in \eqref{prob.primaldual1} now also becomes a hidden variable. Nevertheless, it is trivial to
recover the optimal solution to \eqref{prob.primaldual1} by the following representer theorem.

 \begin{theorem}\label{th:representertheo}
 Let $H$ be a RKHS with positive definite symmetric kernel $k(\cdot,\cdot)$. Suppose \eqref{eq:nondegeneratecondition-nu} holds. Then $(\ba_*,r_*, t_*)$  is optimal for \eqref{prob.csslmunc} if and only if $(\gb^*, s_*, t_*)$  is optimal for \eqref{prob.primaldual1} with
 $r_* = \|\ba_*\|^2- s_*$ and
\begin{equation*} \label{eq:svddcbata}
\ba_*=\sum_{i=1}^\ell\beta_i^*\Phi(\bx_i).
\end{equation*}
\end{theorem}

\begin{proof}
Define
$$
H_\ell:={\rm span}\{\Phi(\bx_1),\cdots,\Phi(\bx_\ell)\}=\left\{\ba\in H: \ba(\cdot)=\sum_{i=1}^\ell\beta_ik(\bx_i,\cdot),\gb\in\R^\ell\right\}.
$$
We first show that if $(\ba_*,r_*, t_*)$ is optimal for \eqref{prob.csslmunc}, then $\ba_*\in H_\ell$.
By contradiction, suppose that $\ba_*\notin H_\ell$.
 Since $H_\ell$ is a linear subspace of $H$, we can decompose $\ba_*\in H$ as  $\ba_*=\bar\ba_*+\ba_*^\perp$ with $\bar\ba_*\in H_\ell$ and $\bm 0\neq \ba_*^\perp\in H_\ell^{\perp}$ where $H_\ell^{\perp}$ is the orthogonal complements of $H_\ell$.
 Moreover, since $\Phi(\bx_i)\in H_\ell$, the Pythagorean theorem then gives
\[
\|\Phi(\bx_i)-\ba_*\|^2=\|\Phi(\bx_i)-\bar\ba_*\|^2+\|\ba_*^\perp\|^2 \ {\rm for}\ i\in[1,\ell].
\]
Elementary computations show that
\[
g(\ba_*,r_*,t_*)=g(\bar\ba_*,r_*-\|\ba_*^\perp\|^2,t_*)+ \tfrac{\nu}{2}\|\ba_*^{\perp}\|^2.
\]
Thus
$$
g(\ba_*,r_*,t_*)>g(\bar\ba_*,r_*-\|\ba_*^\perp\|^2,t_*)
$$
because $\ba_*^\perp\neq \bm 0$. This violates the global optimality of $(\ba_*,r_*,t_*)$ for \eqref{prob.csslmunc}. Hence  $\ba_*\in H_\ell$.

Conversely, if $(\gb^*, s_*, t_*)$ is optimal for \eqref{prob.primaldual1}, then $(\ba_*,r_*,t_*) = (\sum_{i=1}^\ell\beta_i^*\Phi(\bx_i), \|\ba_*\|^2- s_*,   t_*)$ must be optimal for  \eqref{prob.csslmunc}. Otherwise, there exists  $(\hat \ba, \hat r, \hat t) \ne (\ba_*,r_*,t_*)$ such
that $g(\hat \ba, \hat r, \hat t) < g(\ba_*,r_*,t_*)$. However, we have shown there exists a decomposition  $\hat \ba = \sum_{i=1}^\ell \hat \beta_i\Phi(\bx_i)$ such that $(\hat \gb, \|\hat\ba\|_2^2 - \hat r, \hat t)$ yields a smaller objective \eqref{prob.primaldual1}---a contradiction.  Therefore,  $(\ba_*,r_*,t_*) $ is optimal for \eqref{prob.csslmunc}.
%
%
%
%
%
\end{proof}


 As we show in the next, for other choices of parameter values $\nu, \mu$ and $b$,  although the constraint $r \ge 0$ is needed,     the global optimal solution of \eqref{prob.csslmunc} can still be found by solving convex problems.  This result is summarized in the following theorem. The proof is provided in Appendix \ref{sec:appendB}.

\begin{theorem}\label{prop:SSLM-De}
Suppose   $\nu+\mu \ge m/{\ell}$ is satisfied.

$(\romannumeral1)$ If $\mu=0$, the infimum  of \eqref{prob.unc}   is attained at
	\[
	\Omega : = \Big\{(\hat \ba, r,  t):    0 \le r\le \min_{ y_i=1} \{\|\Phi(\bx_i)-\hat\ba\|^2\},      r+t \le \min_{ y_i=-1} \{\|\Phi(\bx_i)-\hat\ba\|^2\}, (\nu-m/\ell)r=0,t\ge0 \Big\}.  
	\]
In particular,
$
r_* = 0,   t_* = \min\limits_{ y_i=-1} \|\Phi(\bx_i)-\hat\ba\|^2,  \ba_* = \hat \ba
$
is an optimal solution of \eqref{prob.unc}.

$(\romannumeral2)$ If $0 < \mu< m/{\ell}$,
	 $(\ba_*, z_*)$ is optimal for
  \begin{equation}\label{prob.qp} 
	\begin{array}{cl}
		\mathop{\Min}\limits_{\ba\in H, z\in\R,\gx\in\R^{n}} & \frac{\ell\lambda}{2}\|\ba\|^2+\frac{\ell\mu}{2} z-  m\langle  \hat\ba, \ba\rangle +b\sum\limits_{i=m+1}^\ell\xi_i\\
		\ST & -\frac{k_{ii}}{2}+ \langle \Phi(\bx_i), \ba\rangle -\frac{z}{2} - \xi_i\le 0, \   \xi_i\ge 0, i\in [m+1,\ell]
	\end{array}
\end{equation}
with $\lambda=m/\ell-\mu$,
	then the solution satisfies $\|\ba_*\|^2 \geq z_*$, and $(\ba_*, 0, \|\ba_*\|^2 - z_*)$ is optimal for \eqref{prob.unc}.
	Conversely, if $(\ba_*, r_*, t_*)$ is optimal for \eqref{prob.unc},   $(\ba_*, \| \ba_* \|^2 - (r_* + t_*))$ is optimal for \eqref{prob.qp}.

$(\romannumeral3)$  If $ \mu  =   m/\ell$,
	 $(\ba_*, z_*)$ is optimal for
 \begin{equation}\label{prob.lp}
\begin{array}{cl}
\mathop{\Min}\limits_{a\in H, z \in \mathbb{R}, \gx\in\mathbb{R}^n} &  \frac{\ell\mu}{2} z- m\langle \hat\ba, \ba\rangle + b\sum\limits_{i=m+1}^l \xi_i \\
\ST & -\frac{k_{ii}}{2}+ \langle \Phi(\bx_i), \ba\rangle -\frac{z}{2} - \xi_i\le 0, \   \xi_i\ge 0, i\in [m+1,\ell],
\end{array}
\end{equation}
	then the solution satisfies $\|\ba_*\|^2 \geq z_*$, and $(\ba_*, 0, \|\ba_*\|^2 - z_*)$ is optimal for \eqref{prob.unc}.
	Conversely, if $(\ba_*, r_*, t_*)$ is optimal for \eqref{prob.unc},  $(\ba_*, \| \ba_* \|^2 - (r_* + t_*))$ is optimal for \eqref{prob.lp}.
\end{theorem}

{
The margin parameter $\mu$ must be chosen carefully to balance the effects of the margin and hinge-loss terms. First, 
$\nu + \mu \ge m/\ell$ ensures the combined effect of volume and margin terms can be supported by the positive examples, preventing the objective from decreasing without bound. Second, $0<\mu < m/\ell$  guarantees that the margin term contributes to discrimination while remaining small enough for the positive hinge losses to counterbalance it. Finally, $\mu = m/\ell$ represents a critical point where the margin term is exactly balanced by the positive examples’ contribution, marking the boundary beyond which the objective would become unbounded.
}

The entire scheme of the proposed CSSLM is described in the flowchart in \Cref{scheme.csslm}.
 The procedure starts with choosing the appropriate parameter values of $\mu, \nu$ and then decides which convex problems to solve, or the solution is trivially obtained.
 Figure \ref{fig:parameters_distribuation} depicts the optimal $r_*$ and $t_*$ of \eqref{prob.unc} dependent on different choices of $\mu$ and $\nu$.

 \begin{figure}[h]
\begin{tikzpicture}[node distance=2cm]
\node (start) [draw] {Input $(\nu, \mu, b)$};
\node (para1) [draw, rounded corners, right of=start,xshift=2cm] {$\mu \le \min\{m/{\ell},bn/{\ell}\}$};
\node (inf) [draw, below of=para1] {return $-\infty$};
\node (para2) [draw, rounded corners, right of=para1,xshift=2cm] {$\mu+\nu < m/{\ell}$};
\node (para3) [draw, rounded corners, below of=para2,xshift=1.5cm,yshift=0.5cm] {$0<\mu <m/{\ell}$};
\node (para4) [draw, rounded corners, below of=para3,yshift=1cm] {$\mu = m/{\ell}$};
\node (para5) [draw, rounded corners, below of=para4,yshift=1cm] {$\mu = 0$};
\node (prob1) [draw, right of=para2, xshift=2.0cm] {solve QP   \eqref{prob.primal1}};
\node (prob2) [draw, right of=para3, xshift=0.8cm] {solve QP  \eqref{prob.qp}};
\node (prob3) [draw, right of=para4, xshift=0.8cm] {solve LP  \eqref{prob.lp}};
\node (prob4) [draw, right of=para5, xshift=0.5cm] {return $  (\hat \ba, 0, 0) $};

\draw [->] (start) -- (para1);
\draw [->] (para1) -- node[anchor=east] {No} (inf);
\draw [->] (para1) -- node[anchor=south] {Yes} (para2);

\draw [->] (para2) -- node[anchor=south] {Yes} (prob1);
\draw [->] (para2) --node[anchor=east] {No} (8,-1.5)--(para3);
\draw [->] (8,-1.5)--(8,-2.5)--(para4);
\draw [->] (8,-2.5)--(8,-3.5)--(para5);

\draw [->] (para3)--(prob2);
\draw [->] (para4)--(prob3);
\draw [->] (para5)--(prob4);
\end{tikzpicture}
 \caption{The flowchart of CSSLM.} \label{scheme.csslm}
\end{figure}

\begin{figure}[h]
	\centering
	\includegraphics[width=0.6\linewidth]{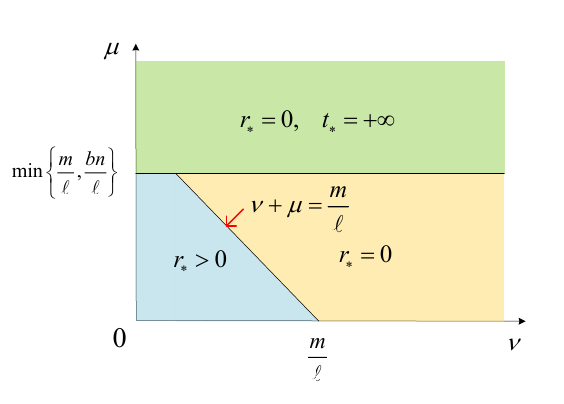}
	\caption{the optimal $r_*$ and $t_*$ of \eqref{prob.unc} dependent on different choices of $\mu$ and $\nu$.}
	\label{fig:parameters_distribuation}
\end{figure}

\section{Connection to SSLM and SVDD}\label{sec.connect}
We have proved that our proposed method can find the \emph{global} minimizer $(\ba_*, r_*, t_*)$ of \eqref{prob.unc}  for  different scenarios of parameters
   $\nu$, $\mu$ and $b$. Now, we show our proposed model can easily render the \emph{global} optimal solutions of SSLM \eqref{prob.sslm} and SVDD (including the one-class
   soft-margin SVDD \eqref{prob.softsvdd} and the two-class soft-margin SVDD \eqref{prob.negativesvdd}).

Letting  $r=R^2$   (or, $R^2 = \|\ba\|^2 - s$ in $f(\ba, s, t)$)  and $  t=\rho^2$,   \eqref{prob.unc} reverts to
\begin{equation}\label{prob.unc1}  
\begin{aligned}
\min\limits_{\ba\in H, R, \rho\in\R }\   \tfrac{\nu}{2}\left(R^2 -\tfrac{\mu}{\nu} \rho^2   +   \tfrac{1}{\nu\ell}\sum\limits_{i=1}^m(d_i^2-R^2)_+ +\tfrac{b}{ \nu\ell}\sum\limits_{i=m+1}^\ell( R^2 + \rho^2  -d_i^2)_+\right)
\end{aligned}
\end{equation}
with $d_i=\|\Phi(\bx_i)-\ba\|$. It is the unconstrained version of the SSLM problem \eqref{prob.sslm} with parameters chosen  as
\begin{equation}\label{eq:parameter2sslm}
(\bar{\nu}, \bar{\nu}_1,\bar{\nu}_2) = \left(\frac{\mu}{\nu}, \frac{\ell\nu}{m}, \frac{\ell\nu}{nb}\right)
\end{equation}
where the objective   of \eqref{prob.unc1} is   that of \eqref{prob.sslm} scaled by $\nu/2$.
 Since \eqref{prob.unc1} and \eqref{prob.sslm}  share the same global minimizers, then $(\ba_*, R_*, \rho_*)$ is the \emph{global} minimizer of \eqref{prob.sslm} with
\begin{equation}\label{eq:rt2R-rho}
 R_* = \sqrt{r_*} \text{ and }  \rho_*  = \sqrt{ t_*}.
\end{equation}

On the other hand, setting  $\mu=0, C = 1/(\nu \ell)$ and $D = b/(\nu \ell)$  in \eqref{prob.unc1}  yields
\begin{equation}\label{prob.unc2} 
\begin{array}{cl}
 \min\limits_{\ba\in H, R, \rho\in\R }  \hat g(\ba, R, \rho)  :=\frac{\nu}{2}\left( R^2  + C \sum\limits_{i=1}^m(d_i^2-R^2)_+ +D\sum\limits_{i=m+1}^\ell{(R^2 + \rho^2  -d_i^2)}_+\right).
\end{array}
\end{equation}
If $m=\ell$,  then the last summation in the objective function of \eqref{prob.unc2} disappears and
 \eqref{prob.unc2} is obviously equivalent to the one-class soft-margin SVDD problem \eqref{prob.softsvdd}.  Hence, $(\ba_*,\sqrt{r_*})$ is the \emph{global} minimizer of \eqref{prob.softsvdd}.
If  $m< \ell$,  we show in the following theorem that \eqref{prob.unc2} is equivalent to the unconstrained version of the two-class soft-margin SVDD \eqref{prob.negativesvdd}, which can be rewritten as
\begin{equation}\label{prob.unc3}
 \min\limits_{\ba\in H, R\in\R }\  \hat g(\ba, R, 0).
\end{equation}

\begin{theorem}\label{thm.svdd}
	Suppose $\mu = 0$ and $m < \ell$.
For any $(\ba_*, R_*, \rho_*)$ optimal for \eqref{prob.unc2},  $(\ba_*, R_*)$ is globally optimal for   \eqref{prob.unc3}.  Conversely,   for any $(\ba_*, R_*)$   optimal for \eqref{prob.unc3},
$  (\ba_*, R_*, 0)$ is globally optimal for \eqref{prob.unc2}.
\end{theorem}

\begin{proof}
For any optimal solution $(\ba_*, R_*, \rho_*)$ for \eqref{prob.unc2},
$ \hat g(\ba_*, R_*, \rho_*)   \ge \hat g(\ba_*, R_*, 0)$ since $\hat g$ is increasing in $|\rho|$.
Therefore, $(\ba_*, R_*, 0)$ is also optimal for \eqref{prob.unc2}.  It follows that
 $\hat g(\ba, R, 0)   \ge \hat g(\ba_*, R_*, 0)$ for any $(\ba, R)$, meaning $(\ba_*, R_*)$ is global
 optimal for \eqref{prob.unc3}.

 Conversely,  if $(\ba_*, R_*)$ is optimal for \eqref{prob.unc3}, then
 \[   \hat g(\ba_*, R_*, 0)  \le \hat g(\ba, R, 0)   \le \hat g(\ba, R, \rho)\]
 for any $(\ba, R, \rho)$. Therefore, $(\ba_*, R_*, 0)$ is optimal for \eqref{prob.unc2}.
  \end{proof}

Theorem \ref{pro:non_degenerate} and Theorem~\ref{thm.svdd} imply that when setting $\mu = 0$, $m<\ell$ and $\nu<m/{\ell}$,  our proposed CSSLM  renders the \emph{global} minimizer $(\ba_*, \sqrt{ r_*})$ of \eqref{prob.negativesvdd}
by obtaining an optimal $(\ba_*, r_*, t_*)$ to  \eqref{prob.csslmunc}.

\begin{remark}
 Theorem~\ref{thm.svdd} does not imply that any optimal $\rho_*$ for  \eqref{prob.unc2}  must be zero.
In fact, if there exists $i_* \in [m+1, \ell]$ such that the optimal $(\ba_*, R_*, \rho_*)$ for \eqref{prob.unc2} satisfies
\begin{equation}\label{ass.point}  \|\Phi(\bx_{i_*})-\ba_*\| \le R_*, \end{equation}
then $\hat g(\ba_*, R_*, 0) < \hat g(\ba_*, R_*, \rho)$
for any $\rho > 0$; therefore,  the optimal solution satisfies  $\rho_* = 0$.
On the other hand, if \eqref{ass.point} is not satisfied, meaning
\[ \|\Phi(\bx_{i})-\ba_*\| > R_*, \quad \text{ for any }  i \in [m+1, \ell],\]
then  $(\ba_*, R_*, \rho_*)$   is   globally optimal for \eqref{prob.unc2}  for any $ \rho_* \in [0, \hat\rho]$ with $\hat\rho := \min_{i\in[m+1,\ell]}\{\|\Phi(\bx_{i_*})-\ba_*\|-R_*\}$.
\end{remark}

We summarize the connections between our proposed CSSLM and the state-of-the-art SSLM \eqref{prob.sslm}, one-class SVDD \eqref{prob.softsvdd} and two-class SVDD \eqref{prob.negativesvdd} in Figure \ref{tik.connetion}. Therefore, the global minimizers of SSLM and SVDD can be found by applying polynomial algorithms to solve CSSLM.

\begin{figure}[h]
\centering
\begin{tikzpicture}\begin{scope} 
\node[draw] (A) at (0,0) {$(\ba_*, r_*, t_*)$};
\node[draw] (A1) at (2,1) {$\mu > 0$};
\node[draw] (A2) at (2,-1) {$\mu = 0$};
\node[draw, rounded corners]                   (B) at (5,1) {$(\ba_*, \sqrt{r_*}, \sqrt{t_*})$};
\node[draw, rounded corners]                   (C) at (5,-1) {$(\ba_*, \sqrt{r_*})$};
\node[draw, rounded corners]                   (C1) at (7,-0.5) {$m = \ell $};
\node[draw, rounded corners]                   (C2) at (7,-1.5) {$m < \ell $};
\node[draw] (B') at (12,1) { SSLM};
\node[draw] (C1') at (12,-0.5) {One-class SVDD};
\node[draw] (C2') at (12,-1.5) {Two-class SVDD};

\draw [->] (A) --(0,1)-- (A1);
\draw [->] (A) --(0,-1)-- (A2);
\draw [->] (A1) -- (B);
\draw [->] (A2) -- (C);
\path [->] (B) edge node[above] { Global optimal } (B');
\draw [->] (C) --(5,-0.5)--(C1);
\draw [->] (C) --(5,-1.5)--(C2);
\path [->] (C1) edge node[above] { Global optimal }  (C1');
\path [->] (C2) edge node [above] { Global optimal } (C2');
\end{scope}
\end{tikzpicture}
\caption{Connection to SSLM and SVDD}\label{tik.connetion}
\end{figure}

\section{Existence and Uniqueness of Optimal Solutions}\label{sec:unique_of_solution}

In this section, we state and prove our main results on the existence and uniqueness of
optimal solutions of the problems \eqref{prob.primal1},      \eqref{prob.qp} and \eqref{prob.lp}  involved in our proposed model.

\subsection{Existence of optimal solutions}

One possible and rather common approach for solving the convex optimization problem is to solve its dual problem. Elementary computation show that the Lagrangian dual of \eqref{prob.primal1} admits an explicit form, i.e.
\begin{equation}\label{prob.dual1}\tag{$\mathcal{D}$}
	\begin{array}{ll}
		\mathop{\Max}\limits_{\ga\in\R^\ell} &   -\frac{1}{2\ell\nu}\sum\limits_{i=1}^{\ell}\sum\limits_{j=1}^\ell y_iy_{j}k_{ij}\alpha_i\alpha_j+\frac{1}{2}\sum\limits_{i=1}^\ell y_ik_{ii}\alpha_i\\
		\ST & \sum\limits_{i=1}^\ell y_i\alpha_i=\ell\nu,\quad \sum\limits_{i=m+1}^\ell\alpha_i\ge \ell\mu,\\
	    	& 0\leq\alpha_i\leq 1, i\in [1,m];\quad  0\leq\alpha_i\leq b,i\in [m+1,\ell].
	\end{array}
\end{equation}

\begin{theorem}\label{th:sslmCoSdualgap}
Suppose \eqref{eq:nondegeneratecondition-nu} holds. Problem \eqref{prob.primal1} and its dual problem \eqref{prob.dual1}
are solvable and the strong duality holds at the primal-dual optimal solution. In particular, $(\ba,s,t,\gx, \ga)$ is the primal-dual optimal solution if and only if it satisfies
\begin{subequations}\label{kkt.1}
	\begin{align}
	\|\Phi(\bx_i)-\ba\|^2-r-2\xi_i\le 0, \xi_i\ge 0,&\ \ i\in [1,m],\label{eq:sslmkktpf-r}\\
	q-\|\Phi(\bx_i)-\ba\|^2-2\xi_i\le 0, \xi_i\ge 0, &\ \ i\in [m+1,\ell], \label{eq:sslmkktpf-q}\\
	\ba=\frac{1}{\ell\nu}\sum_{i=1}^\ell y_i\alpha_i\Phi(\bx_i),& \label{prob.dual1F}\\
	\sum_{i=1}^\ell y_i\alpha_i=\ell\nu,\,\, \sum_{i=m+1}^\ell \alpha_i\ge \ell\mu,& \label{prob.dual1F3}\\
	0 \le \alpha_i \le 1,  \ i\in [1,m], \qquad \  0 \le  \alpha_i \le b,& \ i\in [m+1,\ell], \label{prob.dual1F1}\\
	\xi_i (1- \alpha_i)=0, \  \ i\in [1,m], \quad \xi_i (b- \alpha_i)=0,&\   i\in  [m+1,\ell], \label{eq:sslmrelaxation2}\\
	( \|\Phi(\bx_i)-\ba\|^2-r-2\xi_i )   \alpha_i=0,&\ \ i\in [1,m], \label{eq:sslmkktcs-r}\\
	( q-\|\Phi(\bx_i)-\ba\|^2-2\xi_i)  \alpha_i=0,&\ \ i\in [m+1,\ell], \label{eq:sslmkktcs-q}\\
	t\ge 0,  t\big(\sum_{i=m+1}^\ell\alpha_i-\ell\mu\big) =0&    \label{eq:sslmrelaxation3}
	\end{align}
\end{subequations}
	with $r  =  \|\ba\|^2 - s$ and $q=r+t$.
\end{theorem}
\begin{proof} Problem \eqref{prob.primal1} is clearly feasible and convex. The dual problem \eqref{prob.dual1} is also a convex quadratic program. Moreover, define $\alpha_i=\ell(\nu+\mu)/m$ if $i\in [1,m]$; Otherwise $\alpha_i=\ell\mu/n$.
Then this $\ga$ is feasible to \eqref{prob.dual1} since \eqref{all.value} and \eqref{eq:nondegeneratecondition-nu}.
Owing to the well-known Frank-Wolfe theorem (\citeauthor{FrankWolfe56},\citeyear{FrankWolfe56}), both the quadratic programming problem \eqref{prob.primal1} as well as its dual problem \eqref{prob.dual1} are solvable. By  Proposition 3.4.2 (a) in \cite{Bertsekas99}, strong duality holds at the optimal primal-dual solution. The necessary and sufficient optimality conditions can be derived by applying the Karush-Kuhn-Tucker optimality conditions to \eqref{prob.primal1}.
\end{proof}

As for problem \eqref{prob.qp},  its Lagrangian dual problem admits the following explicit form
\begin{equation}\label{prob.dual2}
	\begin{array}{ll}
		\mathop{\Max}\limits_{\ga\in\R^n} & -\frac{1}{2 \ell \lambda} \sum\limits_{i,j=m+1}^{\ell} k_{ij}\alpha_i\alpha_j -\frac{1}{2} \sum\limits_{i=m+1}^\ell\left(k_{ii}-\frac{2}{l \lambda}\sum\limits_{j=1}^m k_{ji}\right)\alpha_i  - \frac{1}{2 \ell \lambda}\sum\limits_{i,j=1}^mk_{ij}\\
		\ST &   \sum\limits_{i=m+1}^\ell \alpha_i = \ell\mu,\\
		&    0\leq\alpha_i\leq b, i\in [m+1,\ell].
	\end{array}
\end{equation}
 We show in the next theorem that the choice of parameters guarantee that the optimal solutions of \eqref{prob.qp} always exist.

\begin{theorem}\label{thm:solvable2}
	Suppose $\mu+\nu \ge m/{\ell}$ and $\mu < m/{\ell}$  holds.
	Problem \eqref{prob.qp} and its dual problem \eqref{prob.dual2}
	are solvable and the strong duality holds at the primal-dual optimal solution. In particular, $(\ba,z,\gx, \ga)$ is the primal-dual optimal solution if and only if it satisfies
	\begin{subequations}\label{kkt.2}
		\begin{align}
			\ba=   \frac{1}{\ell\lambda}\left[ \sum_{i=1}^m\Phi(\bx_i) -     \sum_{i=m+1}^\ell \alpha_i\Phi(\bx_i)\right],& \label{prob.dual1F2}\\
			\sum_{i=m+1}^\ell \alpha_i = \ell\mu,& \label{prob.dual1F32}\\
			t-\|\Phi(\bx_i)-\ba\|^2-2\xi_i\le 0,\ \xi_i \ge 0, & \    i\in [m+1,\ell], \label{eq:sslmkktpf-q.2}\\
			\xi_i (b- \alpha_i)=0,  \    0 \le  \alpha_i \le b,& \ i\in [m+1,\ell], \label{prob.dual1F12}\\
			( t-\|\Phi(\bx_i)-\ba\|^2-2\xi_i)  \alpha_i=0,&\  i\in [m+1,\ell] \label{eq:sslmkktcs-q2}
		\end{align}
	\end{subequations}
	with $t  =  \|\ba\|^2 - z$.
\end{theorem}

\begin{proof}
	First of all, $\lambda=m/\ell-\mu>0$ implies \eqref{prob.qp} is well defined and convex since $\mu < m/{\ell}$. Moreover,
	define $\alpha_i=
	\ell\mu/{n},  i\in [m+1,\ell]$.
	Then this $\ga$ is feasible to \eqref{prob.dual2} since   \eqref{all.value}.
	Owing to the well-known Frank-Wolfe theorem (\citeauthor{FrankWolfe56},\citeyear{FrankWolfe56}), both the quadratic programming problem \eqref{prob.qp} as well as its dual problem \eqref{prob.dual2} are solvable. By Proposition 3.4.2 (a) in \cite{Bertsekas99}, strong duality holds at the optimal primal-dual solution. The necessary and sufficient optimality conditions can be derived by applying the Karush-Kuhn-Tucker optimality conditions to \eqref{prob.qp}.
\end{proof}

 As for problem \eqref{prob.lp}, its Lagrangian dual problem   is given by
\begin{equation}
	\label{prob.lpdual}
	\begin{array}{cl}
		\mathop{\Min}\limits_{\alpha \in \R^n}\; & \sum\limits_{i=m+1}^\ell k_{ii} \alpha_i \\
		\ST  \;                  & \sum\limits_{i=m+1}^\ell \alpha_i  = \ell \mu \\
								 & 0 \leq \alpha_i \leq b,\; i \in [m+1, \ell]\\
								 & \sum\limits_{i=m+1}^\ell \alpha_i \Phi(\bx_i) = \sum\limits_{i=1}^m\Phi(\bx_i).
	\end{array}
\end{equation}
It is trivial that \eqref{prob.lp} is feasiable. If the dual problem \eqref{prob.lpdual} is infeasible,   \eqref{prob.lp}  is unbounded below, neither does problem \eqref{prob.unc}.
A necessary condition for \eqref{prob.lpdual} to be feasible is
$ \mu/b \leq n/\ell$.
However, we show this condition is not sufficient for the existence of the optimal solution to \eqref{prob.lp} with a  counterexample.

\begin{example}
Let $m = n = 1, \ell = 2,\mu=1/2, b=1$, $\bx_1=(1,0)$ with label $y_1=1$ and   $\bx_2=(0,1)$ with label $y_2=-1$. Considering the linear kernel $\Phi(\bx)=\bx$, the dual problem \eqref{prob.lpdual} is infeasible in this case since $(0,\alpha_2)\neq (1,0)$ for any $\alpha_2\in [0,1]$.  Therefore, the corresponding \eqref{prob.lp} is unbounded below.
\end{example}

\subsection{Uniqueness of the optimal sphere center}

In this subsection, we characterize the optimal solution of the proposed \eqref{prob.primal1}. In particular,  we show that the global optimal solution is unique.
 Such a result represents the stark contrast of \eqref{prob.primal1}, and is unrealistic to obtain for traditional methods such as \eqref{prob.negativesvdd} and \eqref {prob.sslm} due to the lack of
convexity.

First of all, the optimal center of the sphere is unique, since  \eqref{prob.primal1} and \eqref{prob.qp}  is partially strictly convex with respect to $\ba$.

\begin{theorem}\label{th-unique_c}
The optimal $\ba_*$ for \eqref{prob.primal1} and \eqref{prob.qp} is unique if $m/\ell>\mu$.
\end{theorem}
\begin{proof} It suffices to show that for any optimal solutions  $(\ba',s',t')$ and $(\ba'', s'',t'')$ for \eqref{prob.csslmunc-s},  it
holds that $\ba'=\ba''$.  Suppose
 by contradiction that  $\ba'\neq \ba''$.
By  convexity,
$$
\bu :=\tfrac{1}{2}(\ba',s',t')+\tfrac{1}{2}(\ba'',s'',t'')$$
is also optimal for \eqref{prob.csslmunc-s}.  It follows  from $\|\tfrac{1}{2}\ba'+\tfrac{1}{2}\ba''\|^2 < \tfrac{1}{2}\|\ba'\|^2+\tfrac{1}{2}\|\ba''\|^2$
that \[ f(\bu) < \tfrac{1}{2}f(\ba',s',t')+\tfrac{1}{2}f(\ba'',s'',t'') = f(\ba',s',t'),\] contradicting with
the optimality of $(\ba',s',t')$. Therefore, the optimal $\ba_*$ for \eqref{prob.primal1} is unique.  It follows from the same argument
that $\ba_*$ is also unique for \eqref{prob.qp} if $m/\ell > \mu$.
\end{proof}

\subsection{Uniqueness of the optimal radius and margin}

Since the sphere center is unique for \eqref{prob.primal1}, now we fix the sphere center and investigate the uniqueness of the radius and margin.  For this purpose, notice that the strong duality holds for \eqref{prob.primal1} and \eqref{prob.dual1}.
A direct result from this property is that for  any optimal
primal variable  $(\ba,s,t,\gx) $ for \eqref{prob.primal1}  and any optimal dual variable $\ga$ for \eqref{prob.dual1},
$(\ba,s,t,\gx,\ga)$ together satisfies the KKT conditions \eqref{kkt.1}.
Therefore, our analysis is based on
 the optimality conditions \eqref{kkt.1}  and the primal optimal solution is characterized via any given optimal dual variable.

Let $\ga^*$ be an optimal solution to the dual problem \eqref{prob.dual1}. Define the following index sets
\[
\begin{aligned}
  \Tcal_+:= & \{i:y_i=1, 0\le \alpha_i^*<1\}, \quad   \Tcal_-:=\{i:y_i=-1, 0\le  \alpha_i^*<b\},\\
 \SVcal_+:= & \{i:y_i=1,  0 < \alpha_i^*\le 1 \},  \   \SVcal_-:=\{i:y_i=-1, 0 < \alpha_i^* \le b \}.
\end{aligned}
\]

\begin{proposition}\label{th:solutionset_rt}
Suppose \eqref{eq:nondegeneratecondition-nu} holds.  Let $\ga^*$ be an optimal solution to the dual problem \eqref{prob.dual1}.
 The unique optimal sphere center for \eqref{prob.primal1} is
\begin{equation}\label{eq:unique-c}
 \ba_*=\frac{1}{\ell\nu}\sum\limits_{i=1}^\ell y_i\alpha_i^*\Phi(\bx_i).
 \end{equation}
 The optimal radius and margin for \eqref{prob.primal1} is
    \begin{equation}\label{eq:prob.optimal1-rt}
    \Gamma_*=\{(r,t) :  r_l\le r\le r_u, q_l\le r+t\le q_u, t\ge 0, t\big(\sum_{i=m+1}^\ell\alpha_i^*-\ell\mu\big)=0\}
    \end{equation}
 with $r_l:=\max\limits_{i\in \Tcal_+}\|\Phi(\bx_i)-\ba_*\|^2$, $r_u:=\min\limits_{i\in \SVcal_+}\|\Phi(\bx_i)-\ba_*\|^2$,  $q_l:=\max\limits_{i\in \SVcal_-}\|\Phi(\bx_i)-\ba_*\|^2$ and $q_u:=\min\limits_{i\in \Tcal_-}\|\Phi(\bx_i)-\ba_*\|^2$, where by custom $\max\emptyset =-\infty$ and $\min\emptyset = \infty$.
\end{proposition}

\begin{proof} By duality, the  primal-dual pair $(\ba,s,t,\gx, \ga^*)$ is optimal if and only if it satisfies the optimality
 conditions in \eqref{kkt.1} by Theorem \ref{th:sslmCoSdualgap}.
Hence, the vector $ \ba_*$ in \eqref{eq:unique-c} is optimal and unique due to \eqref{prob.dual1F} and Theorem \ref{th-unique_c}.

Since $\ba_*$ is unique, we can treat it as fixed value and let $r=\|\ba_*\|^2-s$ and $q=r+t$.
The complementary slackness condition  \eqref{eq:sslmrelaxation2} 
implies that
 $\xi_i = 0$ for $i \in \Tcal_+\cup\Tcal_-$,  meaning
 $\bx_i,  i \in \Tcal_+\cup\Tcal_-$ are separated correctly by the sphere. The complementary slackness conditions   \eqref{eq:sslmkktcs-r}  and \eqref{eq:sslmkktcs-q} imply
 $\|\Phi(\bx_i)-\ba_*\|^2-r-2\xi_i=0,  i\in \SVcal_+$ and $q-\|\Phi(\bx_i)-\ba_*\|^2-2\xi_i=0, i\in \SVcal_-$. Combining those conditions with the primal feasible conditions \eqref{eq:sslmkktpf-r} and \eqref{eq:sslmkktpf-q}, $(r, t)$ are optimal if and only if  $t\ge 0 $ and the following conditions are satisfied:
\begin{align}
 & r \ge \|\Phi(\bx_i)-\ba_*\|^2 \text{ with }  \xi_i = 0,                                             i \in \Tcal_+,  \nonumber\\
 & r \le \|\Phi(\bx_i)-\ba_*\|^2 \text{ with }  2\xi_i= \|\Phi(\bx_i)-\ba_*\|^2-r,     i \in \SVcal_+,  \nonumber\\
& q \le \|\Phi(\bx_i)-\ba_*\|^2 \text{ with } \xi_i = 0,                                                 i \in \Tcal_-, \nonumber\\
 & q \ge \|\Phi(\bx_i)-\ba_*\|^2 \text{ with } 2\xi_i = q - \|\Phi(\bx_i)-\ba_*\|^2,      i \in \SVcal_-.  \nonumber
 \end{align}
   Then $(r,t)$ is optimal if and only if $(r,t)\in \Gamma_*$.
\end{proof}

\begin{remark} Note that $r$ always lies in a nonnegative finite interval. First, the dual feasibility $\sum_{i=1}^\ell y_i\alpha_i^*=\ell \nu$ and $\nu>0$ imply $\SVcal_+ \ne \emptyset$. Hence $r_u\ge 0$. Secondly, if $\Tcal_+ \neq \emptyset$ then $r_l\ge 0$. Otherwise, $\alpha_i^* = 1$ for all $y_i=1$.  By the equality in   \eqref{prob.dual1F3} one has $\sum_{i=m+1}^\ell \alpha_i^* = m - \ell\nu$. Combining it with \eqref{eq:nondegeneratecondition-nu}, $\sum_{i=m+1}^\ell \alpha_i^*>\ell\mu$, we have $\SVcal_- \neq \emptyset$ and $t_*=0$ by the complementarity in \eqref{eq:sslmrelaxation3}. Hence $q_l\ge 0$ and $r\ge \max\{r_l, q_l\}$.
\end{remark}

We can then use the above characterization of $(r, t)$ to derive the following results on the uniqueness of $(r, t)$.

\begin{theorem}\label{th:unique_rt}
Suppose \eqref{eq:nondegeneratecondition-nu} holds.
 If there exist  $i_+\in [1,m]$ with $\alpha^*_{i_+}\in (0,1)$ and $i_{-}\in [m+1,\ell]$ with $\alpha^*_{i_-}\in(0,b)$, then
$ r_* = \|\Phi(\bx_{i_+})-\ba_*\|^2\  \text{ and } \  t_*= \|\Phi(\bx_{i_-})-\ba_*\|^2 - \|\Phi(\bx_{i_+})-\ba_*\|^2 $ are the unique  optimal radius and margin for \eqref{prob.primal1}, respectively.
If $\sum_{i=m+1}^\ell\alpha_i^*>\ell\mu$ then the optimal margin is $0$ and the optimal radii set is
    \begin{equation}\label{eq:prob.optimal1-r}
    \left[\max\{r_l, q_l\}, \min\{r_u, q_u\}\right].
    \end{equation}
    In particular, if there exist $i_*\in [1,m]$ with $\alpha^*_{i_*}\in (0,1)$ or $i_*\in [m+1,\ell]$ with $\alpha^*_{i_*}\in (0,b)$, then $r_*= \|\Phi(\bx_i)-\ba_*\|^2$ is the unique  optimal radius for \eqref{prob.primal1}.
\end{theorem}

\begin{proof}
Under the assumption, one has $i_+\in \Tcal_+\cap\SVcal_+$ and $i_-\in \Tcal_-\cap\SVcal_-$. Then the optimal radii and margin for \eqref{prob.primal1} in \eqref{eq:prob.optimal1-rt} reduces to the single point
\[
\Gamma_*=\{( \|\Phi(\bx_{i_+})-\ba_*\|^2, \|\Phi(\bx_{i_-})-\ba_*\|^2 - \|\Phi(\bx_{i_+})-\ba_*\|^2)\}.
\]
 If $\sum_{i=m+1}^\ell\alpha_i^*>\ell\mu$ then the optimal margin is $0$ by the complementary condition \eqref{eq:sslmrelaxation3}. Bring $t=0$ into \eqref{eq:prob.optimal1-rt}, one has the desired optimal radii set \eqref{eq:prob.optimal1-r}.

If there exist $i_*\in [1,m]$ with $\alpha^*_{i_*}\in (0,1)$ then $i_*\in \Tcal_+\cap\SVcal_+$. Hence the interval in \eqref{eq:prob.optimal1-r} reduces to the point $\|\phi(\bx_{i_*})-\ba_*\|^2$. Similarly if there exists $i_*\in [m+1,\ell]$ with $\alpha^*_{i_*}\in(0,b)$, then $i_*\in \Tcal_-\cap\SVcal_-$. Hence the interval in \eqref{eq:prob.optimal1-r} also reduces to the point $\|\phi(\bx_{i_*})-\ba_*\|^2$. This completes the proof.
\end{proof}

As for problem \eqref{prob.qp}, we can also have the following  similar result.

\begin{proposition}\label{th:solutionset_rt2}
Suppose $0 < \mu < m/{\ell}$ and $\mu+\nu \ge m/\ell$ hold. Let $\ga^*$ be an optimal solution to the dual problem \eqref{prob.dual2}.
The unique optimal sphere center for \eqref{prob.qp} is
\begin{equation}\label{eq:unique-c-de}
 \ba_*= \frac{1}{\ell\lambda}\left(\sum_{i=1}^m\Phi(\bx_i) -     \sum_{i=m+1}^\ell \alpha_i^*\Phi(\bx_i)\right),
 \end{equation}
and the optimal radius is $r_*=0$.
 If $\sum_{i=1}^\ell\alpha_i^*>\ell\mu$ then the optimal margin is $0$. Otherwise, the optimal margin set is
\begin{equation}\label{eq:prob.optimal1-rt2}
    T_*=\{t :\max\limits_{i\in \SVcal_-}\|\Phi(\bx_i)-\ba_*\|^2\le  t\le \min\limits_{i\in \Tcal_-}\|\Phi(\bx_i)-\ba_*\|^2\}.
\end{equation}
If there exist  $i_{-}\in [m+1,\ell]$ with $\alpha^*_{i_-}\in(0,b)$, then
$   t_*= \|\Phi(\bx_{i_-})-\ba_*\|^2$ is the unique  optimal margin for \eqref{prob.qp}.
\end{proposition}

\begin{proof} First of all,   \Cref{thm:solvable2} and \Cref{prop:SSLM-De}  implies  \eqref{prob.qp} is solvable and the optimal radius is $r_*=0$.
By duality, the  primal-dual pair $(\ba_*, t, \gx, \ga^*)$ is optimal if and only if it satisfies the optimality
 conditions in \eqref{kkt.2}.
Hence, the vector $ \ba_*$ in \eqref{eq:unique-c-de} is optimal and unique due to Theorem \ref{th-unique_c}.

Since $\ba_*$ is given and  $r_*= 0$, it follows from the optimality condition in \eqref{prob.dual1F12} and \eqref{eq:sslmkktcs-q2} 
 that    $\xi_i  = 0$ and $\|\phi(\bx_{i})-\ba_*\|^2 \le t$ for any $i\in \Tcal_-$ and $\|\phi(\bx_{i})-\ba_*\|^2 - t - 2\xi_i = 0$ for any $i\in \SVcal_-$.
Therefore, $t$ is optimal if and only if  $t\ge 0 $ and the following conditions are satisfied:
\begin{align}
& t \le \|\phi(\bx_{i})-\ba_*\|^2 \text{ with } \xi_i = 0,   \qquad  i \in \Tcal_-, \label{eq.kkt.cond3}\\
 & t \ge \|\phi(\bx_{i})-\ba_*\|^2 \text{ with } 2\xi_i = t -\|\phi(\bx_{i})-\ba_*\|^2, \     i \in \SVcal_-.  \label{eq.kkt.cond4}
 \end{align}
   Then $t$ is optimal if and only if $t\in T_*$.
  Under the assumption, one has   $i_-\in \Tcal_-\cap\SVcal_-$. Then the optimal margin set $T_*$ in \eqref{eq:prob.optimal1-rt2} reduces to the single point $\{\|\phi(\bx_{i_-})-\ba_*\|^2\}$.
\end{proof}

\subsection{$\nu$-property}
Following the terminology in \cite{Scholkopf01} (see e.g. \citeauthor{Scholkopf00}, \citeyear{Scholkopf00}; \citeauthor{wu2009small}, \citeyear{wu2009small}), A training example $\bx_i, 1\le i\le \ell$  is called a support vector (SV) if the associated  multiplier  $\alpha^*_i>0$. Therefore,   $\{\bx_i: i\in \SVcal_-\cup\SVcal_+\}$ are the set of  SVs.
$\bx_i  $ is a margin error (ME) if the corresponding   $\xi^*_i>0$.
As for MEs, by
complementary condition \eqref{eq:sslmrelaxation2},
\begin{equation}\label{me.dual}
\alpha_i^* =  \begin{cases} 1  & \text{ if $\bx_i$ is a positive ME}, \\
b & \text{ if $\bx_i$ is a negative ME}.
\end{cases}
 \end{equation}
Therefore,  the MEs in the positive  constitute a subset of $\SVcal_+$ and
the MEs in the negative is a subset of $\SVcal_-$.
Let $m_+$ and $n_+$ denote the number of MEs in the positive and negative classes, while $s_+$ and $s_-$ denote the number of SVs in the positive and negative classes, respectively.  This means
\[ m_+ \le | \SVcal_+|, \quad  n_+ \le |\SVcal_-|,\quad s_+ = |\SVcal_+|,\quad   s_- = |\SVcal_-|.\]

The $\nu$-property \citep{Scholkopf99, Scholkopf01, Scholkopf00, wu2009small}
established the connection between the model parameters and the fractions of SV and ME
in the dataset.  Our model \eqref{prob.primal1}  involves a  nonnegative parameter $\mu$ and two positive parameters $\nu$ and $b$.
We can also derive a similar $\nu$-property for these parameters in the following proposition.

\begin{proposition}\label{th:nu-property}
Suppose \eqref{eq:nondegeneratecondition-nu}  holds.
If the optimal margin of \eqref{prob.primal1} is positive, then
\begin{equation}\label{eq:nuproperty1}
\frac{m_+}{\ell}\le\nu+\mu\le \frac{s_+}{\ell}, \quad \frac{n_+}{\ell}\le\frac{\mu}{b}\le \frac{s_-}{\ell};
\end{equation}
Otherwise,
\begin{equation}\label{eq:nuproperty0}
\frac{m_+}{\ell}\le \nu +b\frac{s_-}{\ell},\quad  \nu+b\frac{n_+}{\ell}\le \frac{s_+}{\ell}, \quad \frac{\mu}{b}\le\frac{s_-}{\ell}.
\end{equation}
\end{proposition}

\begin{proof} From the optimality condition \eqref{prob.dual1F3}, there exists an optimal solution $\ga^*$ of \eqref{prob.dual1} satisfies
\begin{equation}\label{eq:depart2}
	\sum_{i=1}^m\alpha^*_i-\sum_{i=m+1}^\ell\alpha^*_i=\ell\nu
\end{equation}
and
\begin{equation}\label{eq:nagativeLM>=}
	\sum_{i=m+1}^\ell\alpha^*_i\ge\ell\mu.
\end{equation}
If $t_*>0$, by \eqref{eq:sslmrelaxation3}, it holds that
\begin{equation}\label{eq:nagativeLM=}
	\sum_{i=m+1}^{\ell}\alpha_i^*=\ell\mu
\end{equation}
and
\begin{equation}\label{eq:nuxz1}
	\sum_{i=1}^m\alpha_i^*=\ell(\nu+\mu).
\end{equation}
Summing up $\alpha^*_i$ over the positive MEs leads to
\begin{equation}\label{eq:nuxz2}
m_+ =  \sum_{\bx_i \text{ is positive ME}} \alpha_i^* \le   \sum_{i\in\SVcal_+} \alpha_i^* =   \sum_{i=1}^m\alpha_i^*,
\end{equation}
where the last equality is by the fact that $\alpha_i^* = 0$ for 
$i\in[1,m]\setminus\SVcal_+$.
One the other hand,
\begin{equation}\label{eq:nuxz3}
s_+ =  |\SVcal_+| \ge \sum_{i\in \SVcal_+  } \alpha_i^*  =    \sum_{i=1}^m\alpha_i^*.
\end{equation}
Combining \eqref{eq:nuxz1}, \eqref{eq:nuxz2} and \eqref{eq:nuxz3} we can derive the first group inequalities in \eqref{eq:nuproperty1}.

Summing up $\alpha^*_i$ over the negative  MEs leads to
\begin{equation}\label{eq:nuxz21}
bn_+ =  \sum_{\bx_i \text{ is negative ME}} \alpha_i^* \le   \sum_{i\in\SVcal_-} \alpha_i^* =   \sum_{i=m+1}^\ell\alpha_i^*,
\end{equation}
where the last equality is by the fact that $\alpha_i^* = 0$ for $i\in[m+1,\ell]\setminus\SVcal_-$. On the other hand,
\begin{equation}\label{eq:nuxz31}
bs_-   \ge \sum_{i\in \SVcal_-  }  \alpha_i^*   =    \sum_{i=m+1}^\ell \alpha_i^*.
\end{equation}
The second group inequalities in \eqref{eq:nuproperty1} is then derived based on \eqref{eq:nagativeLM=},  \eqref{eq:nuxz21} and \eqref{eq:nuxz31}.

If the margin $t_* =  0$, it follows from \eqref{eq:nuxz2}, \eqref{eq:nuxz31} and \eqref{eq:depart2} that
\begin{equation}\label{eq:nuxz5}
m_+-bs_-\le
 \sum_{i=1}^m\alpha^*_i-\sum_{i=m+1}^\ell\alpha^*_i=\ell\nu.
\end{equation}
Similarly, it follows from \eqref{eq:nuxz3}, \eqref{eq:nuxz21} and \eqref{eq:depart2}  that
\begin{equation}\label{eq:nuxz4}
s_+-bn_+ \ge  \sum_{i=1}^m\alpha^*_i-\sum_{i=m+1}^\ell\alpha^*_i=\ell\nu.
\end{equation}
Combining \eqref{eq:nuxz5} and \eqref{eq:nuxz4} yields the first two inequalities in \eqref{eq:nuproperty0}.
The last inequality in \eqref{eq:nuproperty0} is true since by \eqref{eq:nagativeLM>=} one has
\[
bs_- \ge \sum_{i\in\SVcal_-}\alpha^*_i =  \sum_{i=m+1}^\ell\alpha^*_i\ge\ell\mu,
\]
completing the proof.
\end{proof}

In the absence of the negative data where  $\mu  = 0 $ and all the terms regarding the negative data disappear in \eqref{prob.primal1}, then assumption \eqref{eq:nondegeneratecondition-nu}  reduces to $\nu<1$ and one can easily arrive at the following result for this particular case.

\begin{corollary}\label{coro:nu-property}
Assume $\nu\in (0,1)$ and $m=\ell$. Then the solution of \eqref{prob.primal1} satisfies
\begin{equation}\label{eq:nuproperty1.1}
\frac{m_+}{m}\le \nu\le \frac{s_+}{m},
\end{equation}
that is $\nu$ is an upper bound on the fraction of errors, and is   a lower bound on the fraction of SVs.
\end{corollary}

The following result for the degenerate case can also be easily derived by a similar argument based on \eqref{prob.dual1F32}.

\begin{proposition}\label{th:nu-property2}
Suppose $\mu < m/{\ell}$ and $\mu+\nu \ge m/\ell$ holds. One has
 \begin{equation}\label{eq:nuproperty12}
  \frac{n_+}{\ell}\le \frac{\mu}{b}\le \frac{s_-}{\ell}.
\end{equation}
 \end{proposition}

\begin{proof}  Summing up $\alpha^*_i$ over the negative MEs leads to
\begin{equation}\label{eq:nuxz22}
bn_+ =  \sum_{\bx_i \text{ is negative ME}} \alpha_i^* \le   \sum_{i\in\SVcal_-} \alpha_i^* \le   \sum_{i=m+1}^\ell \alpha_i^* = \ell\mu
\end{equation}
by  \eqref{prob.dual1F32}  and  \eqref{me.dual}.
One the other hand,
\begin{equation}\label{eq:nuxz32}
bs_-\ge \sum_{i\in \SVcal_-  } \alpha_i^*  =    \sum_{i=m+1}^\ell\alpha_i^*  = \ell\mu.
\end{equation}
Combining \eqref{eq:nuxz22} and \eqref{eq:nuxz32} we can derive   \eqref{eq:nuproperty12}.
\end{proof}

\section{Numerical Experiments}\label{sec.test}


In this section, we report experimental results to demonstrate the effectiveness of the proposed method. 
We investigate the influence of nonconvexity in state-of-the-art models and compare the classification boundaries learned by our proposed primal CSSLM \eqref{prob.primal1}. We also compare the performance of the  dual CSSLM \eqref{prob.dual1}--against several baselines, including  primal SSLM \eqref{prob.sslm},   dual SSLM \eqref{prob.sslmdual1},  SVDD, SVM \eqref{eq:svm}, Decision Tree (DT), KNN, 
 and DNN models (MLP and CNN). 

Most experiments are executed on a machine with an Intel Core i7-13700H CPU and 16 GB of RAM, while the DNN models are trained on a separate server with an Intel Xeon Gold 6330 CPU, 100 GB of RAM, and an NVIDIA RTX 3090 GPU.


\subsection{Nonconvexity issues in solving the primal and dual SSLM}

 We analyze the influence of nonconvexity in primal SSLM  \eqref{prob.sslm} and dual SSLM \eqref{prob.sslmdual1} respectively with the Banana Dataset\footnote{\href{https://leon.bottou.org/papers/bordes-ertekin-weston-bottou-2005}{https://leon.bottou.org/papers/bordes-ertekin-weston-bottou-2005}}, which is a synthetic 2D benchmark containing 5,300 points forming two banana-shaped clusters. Approximately 10\% of points lie in an overlapping noisy region. To reduce noise, we relabel points satisfying 
$x_2 + \tfrac{3}{7}(x_1 - 3) > 0$ as $-1$.  The dataset is generated from Gaussian data nonlinearly transformed into separable banana classes, providing an intuitive yet challenging test for classifier performance. Four scenarios are randomly chosen from the training set, and Scenarios A-C address binary classification while Scenario D focuses on anomaly detection. Note that to ensure the reproducibility of the experimental results, once the data for each scenario are randomly selected from the training set, they become fixed. Table \ref{tab.data} is the summary of the banana data set.

\begin{table}[h!]
	\centering
	\caption{Summary of training sets across scenarios and test set of the Banana Dataset.}
		\label{tab.data}
	\begin{tabular}{cccccccc}\hline\hline
		& Training set&  Scenario A &  Scenario B &  Scenario C & Scenario D  & Test \\ \hline
		$(m,n)$&	$(1630, 2670)$ & $(3,3)$   & $(100,100)$  & $(768,1232)$   & $(1000,100)$ & $(381,619)$\\ \hline\hline
	\end{tabular}
\end{table}

To illustrate the challenges of solving the primal and dual SSLM due to its nonconvexity, we applied two advanced nonlinear solvers available in \texttt{AMPL} to the 10-dimensional feature space in Scenarios A, B, and C using the cubic polynomial feature mapping:

$$
\Phi(\bx)  = (1, \sqrt{3}x_1, \sqrt{3}x_2, \sqrt{3}x_1^2, \sqrt{3}x_2^2, \sqrt{6}x_1x_2, x_1^3, x_2^3, \sqrt{3}x_1x_2^2, \sqrt{3}x_1^2x_2),
$$
corresponding to the kernel function $k(\bx, \by) = (1 + \bx^T\by)^3$.

For consistent comparison, the optimal values of SSLM and dual SSLM were scaled by a factor of $\ell \nu/2$ to match the optimal values of the primal CSSLM \eqref{prob.primal1} and dual CSSLM \eqref{prob.dual1}. The parameters were set as $b=1,\nu=0.001$ and $\mu=0.05$, satisfying conditions \eqref{all.value} and \eqref{eq:nondegeneratecondition-nu}. SSLM-specific parameters 
$(\bar{\nu}, \bar{\nu}_1, \bar{\nu}_2)$ were configured according to \eqref{eq:parameter2sslm}.

{The solvers employed include the global optimization solver \texttt{Gurobi} and the local optimization solver \texttt{Snopt}. Table~\ref{table:ampl_for_sslm} summarizes some initial points, optimal values, runtime, and optimality residuals obtained by each solver. Figure \ref{fig:random initial points} illustrates the optimal values of the primal SSLM \eqref{prob.sslm} obtained with random initializations under the conditions $R=0, \rho=0$ and $R=0, \rho \neq 0$ in Scenario C on the Banana Dataset.}

{
\begin{table}[h!]
    \centering
 \caption{\small Initial points and results of \texttt{Gurobi} and \texttt{Snopt} for the primal SSLM \eqref{prob.sslm} in the 10-dimensional feature space of Scenarios A, B, and C on the Banana Dataset.}
    \label{table:ampl_for_sslm}
    \begin{tabular}{cccrrc}
    \hline
    \hline
    Scenario & Initial point $(\ba, R, \rho)$ & Solver & Optimal value & Runtime  & Optimality Error\\ 
    \hline
    \multirow{6}{*}{A} & \multirow{2}{*}{$(\mathbf{0}, 0, 0)$} & \texttt{Gurobi} &   ${\bf -225.8368}$  &      $4.79$      &    9.4735e-9  \\
    & &\texttt{Snopt}  &    $1.3246$           &   $0.00$ &   7.0000e-7           \\
    \cline{2-6}
    & \multirow{2}{*}{$(\mathbf{0}, 1, 0)$} & \texttt{Gurobi} &   ${\bf -225.8368}$  &      $3.81$      &    9.4735e-9  \\
    & &\texttt{Snopt}  &    $0.0036$           &   $0.00$ &   1.7000e-6           \\
    \cline{2-6}
    & \multirow{2}{*}{$(\mathbf{1}, 0, 1)$} & \texttt{Gurobi} &   ${\bf -225.8368}$  &      $3.71$      &    9.4735e-9  \\
    & &\texttt{Snopt}  &    $-9.9497$           &   $0.00$ &   1.1000e-6          \\
    \cline{2-6}
    \cline{2-6}
    \hline
    \multirow{6}{*}{B} & \multirow{2}{*}{$(\mathbf{0}, 0, 0)$} & \texttt{Gurobi} &   ${\bf 121.9522}$  &      $0.96$      &    1.1057e-8  \\
    & &\texttt{Snopt}  &    $1586.8200$           &   $0.03$ &   1.3000e-7           \\
    \cline{2-6}
    & \multirow{2}{*}{$(\mathbf{0}, 1, 0)$} & \texttt{Gurobi} &   ${\bf 121.9522}$  &      $0.83$      &    1.1057e-8 \\
    & &\texttt{Snopt}  &    ${\bf 121.9522}$         &   $0.01$ &   1.8000e-7           \\
    \cline{2-6}
    & \multirow{2}{*}{$(\mathbf{1}, 0, 1)$} & \texttt{Gurobi} &   ${\bf 121.9522}$  &      $0.80$      &    1.1057e-8 \\
    & &\texttt{Snopt}  &    $1546.7500$           &   $0.07$ &   1.7000e-6           \\
    \cline{2-6}
    \cline{2-6}
    \hline
    \multirow{6}{*}{C} & \multirow{2}{*}{$(\mathbf{0}, 0, 0)$} & \texttt{Gurobi} &   ${\bf 1747.4100}$  &      $13.13$      &    6.0332e-9  \\
    & &\texttt{Snopt}  &    $12024.7000$           &   $2.78$ &   5.8000e-7           \\
    \cline{2-6}
    & \multirow{2}{*}{$(\mathbf{0}, 1, 0)$} & \texttt{Gurobi} &   ${\bf 1747.4100}$  &      $13.32$      &    6.0332e-9  \\
    & &\texttt{Snopt}  &    ${\bf 1747.4100}$           &   $0.57$ &   6.5000e-11           \\
    \cline{2-6}
    & \multirow{2}{*}{$(\mathbf{1}, 0, 1)$} & \texttt{Gurobi} &   ${\bf 1747.4100}$  &      $13.66$      &    6.0332e-9  \\
    & &\texttt{Snopt}  &    $11679.5000$           &   $3.68$ &   3.3000e-7           \\
    \cline{2-6}
    \cline{2-6}
    \hline
    \hline
    \end{tabular}	
\end{table}
}

\begin{figure}[h!]
	\centering
	\subfigure[$R=0, \rho=0$]{
		\begin{minipage}[t]{0.45\linewidth}
			\centering
			\includegraphics[width=\linewidth]{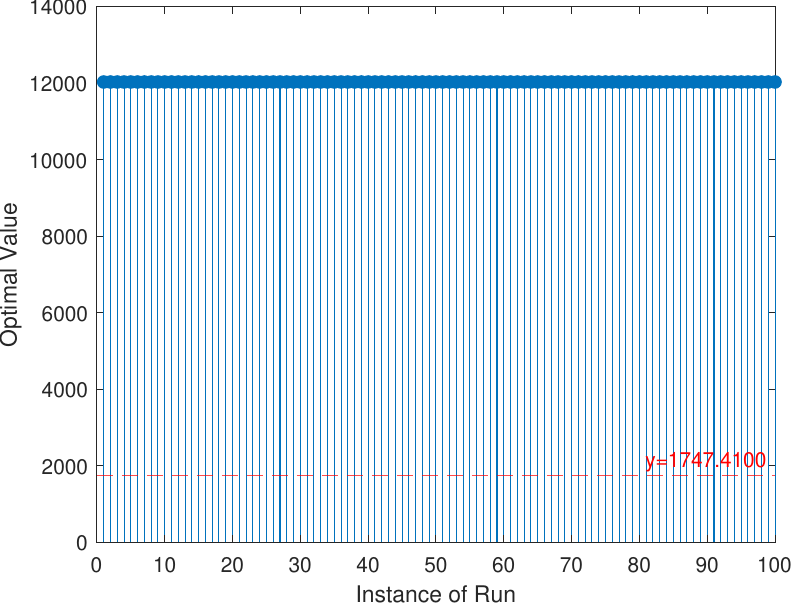}
			\label{fig:R=0 rho=0}			
		\end{minipage}
	}
    \hfill
	\subfigure[$R=0, \rho \neq 0$]{
		\begin{minipage}[t]{0.45\linewidth}
			\centering
			\includegraphics[width=\linewidth]{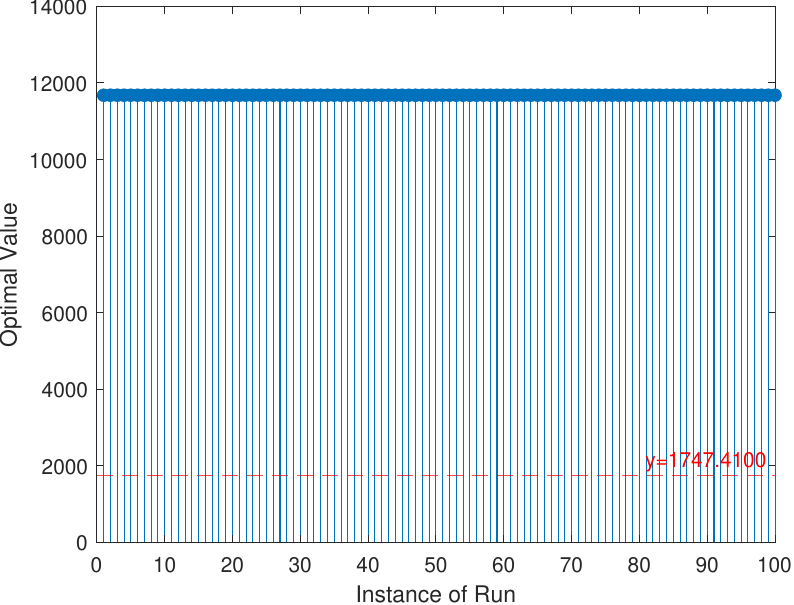}
			\label{fig:R=0}		
		\end{minipage}
	}
\caption{\small 
Optimal values of the primal SSLM \eqref{prob.sslm} under 100 random initializations on the Banana Dataset (Scenario C) with $R=0, \rho=0$ and $R=0, \rho\ne 0$. The dashed horizontal line marks the global optimal value ($y = 1747.41$), which was NOT achieved in the vast majority of runs, demonstrating the sensitivity  of the algorithm to initial conditions.
}
	 \label{fig:random initial points}
\end{figure}

{
Our experiments yield several key observations. Across all three scenarios,  \texttt{Gurobi} successfully identified the global optimum (highlighted in bold), though it required substantially more computational time compared to local solvers such as  \texttt{Snopt}. Notably,  \texttt{Snopt} consistently converged to a local optimum from different initial points, and the significant gap between local and global solutions underscores the risk of relying solely on local methods. Furthermore, under certain random initializations, the primal SSLM \eqref{prob.sslm} also exhibited convergence to local optima, highlighting the challenge of nonconvex optimization even with advanced solvers.
}


As demonstrated, achieving the global optimum for the nonconvex primal SSLM \eqref{prob.sslm} is challenging. A natural alternative is to solve the convex dual SSLM \eqref{prob.sslmdual1} instead of the primal problem. However, our experiment shows that the dual SSLM does not always yield the global optimal solution, reflecting the lack of strong duality.
We used \texttt{quadprog} in Matlab to solve the dual SSLM \eqref{prob.sslmdual1}, primal CSSLM \eqref{prob.primal1}, and its dual problem  \eqref{prob.dual1} in Scenarios A, B, and C on Banana Dataset, all of which are convex QPs. Table \ref{table:sslm_dual} lists the optimal objectives, exitflags\footnote{$\texttt{exitflag}=1$ in \texttt{quadprog} means the optimal solution is found}, runtime, and optimality errors. 

\begin{table}[h!]
    \centering
    \caption{\small The results of solvers for the dual SSLM \eqref{prob.sslmdual1}, proposed problem \eqref{prob.primal1} and its dual problem \eqref{prob.dual1} in 10-dimensional feature space of Scenarios A, B, C on the Banana Dataset.}
    \label{table:sslm_dual}
    \begin{tabular}{clrcrc}
    \hline
    \hline
    Scenario & Model & Optimal value & Exitflag & Runtime  & Optimality Error\\
    \hline
    \multirow{3}{*}{A} & Dual SSLM \eqref{prob.sslmdual1}  &    ${\bf -225.8368}$           &    $1$     &     $0.0161$   &    2.7149e-11    \\
     & CSSLM \eqref{prob.primal1}  &     ${\bf -225.8368}$          &     $1$     &       $0.0072$  &  4.1457e-9     \\
     &  CSSLM \eqref{prob.dual1}  &    ${\bf -225.8368}$           &   $1$       &   $0.0048$ &   1.3184e-10          \\
    \hline
    \multirow{3}{*}{B} & Dual SSLM \eqref{prob.sslmdual1}  &    ${\bf 121.7294}$          &    $1$      &     $0.0067$ &   1.7144e-11         \\
    & CSSLM \eqref{prob.primal1}  &     $121.9522$          &     $1$     &       $ 0.0425 $ &   6.5391e-11       \\
    & CSSLM \eqref{prob.dual1} &    $121.9522$           &   $1$       &   $0.0066 $ &   1.9341e-9          \\
    \hline
    \multirow{3}{*}{C} & Dual SSLM \eqref{prob.sslmdual1} &    ${\bf 1733.1000}$          &    $1$      &     $2.0308 $ &   4.1751e-9        \\
    & CSSLM \eqref{prob.primal1}  &     $1747.4000$         &     $1$    &       $16.0189$ &   1.6268e-10     \\
    & CSSLM \eqref{prob.dual1}  &    $1747.4000$          &   $1$       &   $1.4373 $ &    4.0643e-12         \\
    \hline
    \hline
    \end{tabular}	
\end{table}

We observe that the dual optimal solution for SSLM does not always correspond to the global optimal solution of the primal SSLM (in bold) in Scenarios B and C (in bold). This highlights the absence of strong duality in nonconvex SSLM, indicating that solving the dual SSLM does not resolve nonconvexity issues present in the primal problem. Our proposed primal and dual CSSLM consistently achieved the global optimal solution (as verified by the global value determined by \texttt{Gurobi} in Table \ref{table:ampl_for_sslm}), which confirms strong duality in our model, a clear advantage over nonconvex alternatives.

\subsection{Comprehensive Performance Evaluation Across Datasets}
We then compare the CSSLM with other existing methods, which include DT, KNN, SVM, SVDD, on the Multivariate Cauchy (MC) Dataset, Blood Transfusion Service Center (BTSC), Breast Cancer Wisconsin Dataset (BCW) and KDD-Cup'99 Dataset. For each dataset, we consider two scenarios, i.e. the binary classification (BC) and anomaly detection (AD). Table \ref{tab.data-four} gives the summary of each data set and each scenario.

\begin{table}[h!]
	\centering
	\caption{Summary of each data set and each scenario in performance evaluation. }
	\label{tab.data-four}
	\begin{tabular}{lllll}\hline\hline
		  $(m,n)$       &  MC          &  BTSC        & BCW           & KDD-Cup'99 \\ \hline
		Training set  & $(250, 750)$ & $(125, 399)$ & $(149,250)$   & $(97278,396743)$\\             
            Test set      & $(250, 750)$ & $(53, 171)$  & $(63,107)$    & $(60593,250436)$\\
            BC    & $(200, 200)$ & $(100,100)$  & $(100,100)$   & $(1000,1000)$\\
            AD    & $(200, 20)$  & $(100,10)$   & $(100,10)$    & $(1000, 100)$\\
        \hline\hline
	\end{tabular}
\end{table}

As for the evaluation metric, we adopt the Area Under the Precision--Recall Curve (AUPRC) \citep{davis2006relationship, saito2015precision}, which is particularly suitable for imbalanced datasets and emphasizes the performance of the positive class. Following this setting, we compute AUPRC by treating negative samples as the positive class. 

For parameter selection, we employ five-fold stratified cross-validation following \cite{wu2009small}, ensuring that each fold preserves the positive-to-negative ratio of the original training set. Specifically we use grid search to select the hyperparameters that yield the highest average AUPRC across the folds. Table \ref{table:grid-parameters} gives the grids of parameters in each model, where the Gaussian kernel $k(\bx, \by) = e^{-\gamma \| \bx - \by\|^2}$ is used in SVM, SVDD and CSSLM, MinSamplesLeaf and MaxNumSplits are the minimum number of observations per leaf and the maximum number of splits in DT, respectively. For CSSLM, the parameter $\mu$ is selected from $\{0, 0.01m/\ell, 0.05m/\ell, 0.1m/\ell, m/\ell\}$ and delete the value according to \eqref{all.value}. Then the parameter $\nu$ is selected from $\{0.01m/\ell, 0.1m/\ell\} $ according to \eqref{eq:nondegeneratecondition-nu}. 
These constraints significantly reduce the computational overhead typically associated with cross-validation.

\begin{table}[h!]
    \centering
   \caption{ The grid for each parameters.} 
    \label{table:grid-parameters}        
    \begin{tabular}{lll}    \hline    \hline
    Parameter & Grid & Model\\ \hline  
    $\gamma$  & $\{ \tfrac{2^k}{\sigma^2} : k\in [-4,4]\}^\dag$ & SVM, SVDD, CSSLM\\    
    $b$       & $\{\tfrac{m}{4n}, \tfrac{m}{2n}, \tfrac{m}{n},   \tfrac{2m}{n},  \tfrac{4m}{n} \}$ & SVM, SVDD, CSSLM\\
    $\tfrac{1}{\nu\ell}$ & $\left\{0.01, 0.05, 0.1, 0.5, 1, 5, 10, 50, 100, 500 \right\}$  & SVM\eqref{eq:svm}\\
    $\mu$ & $\{0, \tfrac{0.01m}{\ell}, \tfrac{0.05m}{\ell}, \tfrac{0.1m}{\ell}, \tfrac{m}{\ell}\}$ & CSSLM\\
    $\nu$   & $\{\tfrac{0.01k m}{\ell}: k \in [1, 9]\} \cup \{\tfrac{0.1k m}{\ell}: k\in [1, 9]\}$   & SVDD \eqref{prob.unc2}\\
    $\nu$   & $\{\tfrac{0.01m}{\ell}, \tfrac{0.1m}{\ell}\} $   & CSSLM\\
     MinSamplesLeaf      & $\{1,2,5,10,20,50\}$  & DT\\
      MaxNumSplits  &  $\{5,10,20,50,100,200,500\}$ &  DT\\        
    $k$     &    $\{1,2,5,10,20,25,50,100,500\}$  & KNN \\    \hline    \hline
    \end{tabular}\\
   $^\dag$ Here $\sigma^2$ is the mean squared norm of the training data.\qquad\qquad
\end{table}

Test errors on the positive and negative classes are denoted by $e_+$ and $e_-$, respectively. 
We also employ the geometric mean of sensitivity and specificity (G-means) metric, errors on the positive and negative classes as performance measures to show the overall accuracy and the trade-off for false-positive and false-negative rates. G-means is defined as $
g = \sqrt{(1-e_+)(1-e_-)}$,   
where $1-e_+$ and $1-e_-$ are the so-called accuracy on positive and negative class.   

Note that to quantify the central tendency and variability of algorithm performance, the data for each scenario are randomly selected from the training set, and calculate the AUPRC, G-means, and classification errors ($e_+$, $e_-$)  on the test set.
Tables  \ref{table:resultofCPEAD} report the average AUPRC, G-means, and classification errors ($e_+$, $e_-$) with standard deviations across 10 runs.  
The best results in each scenario are highlighted in bold. A method is considered to have broken down if its average G-means falls below 50\%. Below we describe each dataset and summarize dataset-specific findings. 
 
\begin{table}[h!]
    \centering
   \caption{\small Average AUPRC, G-means, test errors ($e_+$, $e_-$) and standard deviations for models on each scenario on each Dataset.}
    \label{table:resultofCPEAD}
        {\scriptsize 
    \begin{tabular}{lclcccc}
    \hline
    \hline
    Dataset &Scenario & Model & AUPRC & G-means $(\%)$ & $e_+(\%)$ & $e_-(\%)$ \\
    \hline
    \multirow{10}{*}{MC}& \multirow{5}{*}{BC} & DT &    $ 0.9821 \pm 0.0106$  &   $ 92.60 \pm 0.92$      &    $4.88 \pm 2.77$     &     $9.79 \pm 2.56$    \\ 
   & & KNN &    $ 0.9918 \pm 0.0016$    &   $ 91.15 \pm 0.92$        &    $\bf{1.12 \pm 0.88}$     &     $15.95 \pm 2.35$    \\
    && SVM &    $ 0.9942 \pm 0.0006$   &   $ \bf{93.39 \pm 0.81}$         &    $2.92 \pm 2.75$     &     $10.07 \pm 3.85$    \\
    & & SVDD &     $0.9930 \pm 0.0010$  &   $ 92.92 \pm 1.30$         &     $8.08 \pm 4.05$     &       $5.97 \pm 2.20$     \\
    & &  CSSLM &    $\bf{0.9944 \pm 0.0002}$   &   $ 92.09 \pm 0.78$         &   $11.92 \pm 3.00$       &   $\bf{3.60 \pm 0.49}$           \\   \cline{2-7}
   & \multirow{5}{*}{AD} & DT &    $ 0.9697 \pm 0.0058$  &   $ 87.14 \pm 2.75$         &    $0.24 \pm 0.34$     &     $23.81 \pm 4.96$    \\ 
  &  & KNN &    $ 0.9807 \pm 0.0063$    &   $ 65.63 \pm 34.66$        &    $\bf{0.00 \pm 0.00}$     &     $46.12 \pm 28.65$    \\ 
   & & SVM  &    $\bf{0.9910 \pm 0.0083}$  &   $ 71.84 \pm 38.19$         &    $0.48 \pm 0.88$      &     $34.85 \pm 35.43$         \\
   & & SVDD &     $0.9892 \pm 0.0036$   &   $ 85.93 \pm 12.83$        &     $16.20 \pm 23.62$     &       $ 9.16 \pm 4.89 $      \\
   & & CSSLM &    $0.9881 \pm 0.0082$   &   $ \bf{87.61 \pm 8.60}$        &   $17.96 \pm 14.95$       &   $\bf{5.59 \pm 2.97} $        \\
    \hline
 \multirow{10}{*}{BTSC}&   \multirow{5}{*}{BC} & DT &    $0.8943 \pm 0.0207$ & $\bf{64.98 \pm 2.89}$          &    $41.87 \pm 10.33$     &     $\bf{26.02 \pm 8.11}$    \\ 
  &  & KNN &    $ 0.8816 \pm 0.0002$   & $6.61 \pm 20.90$           &    $3.02 \pm 9.55$     &     $93.74 \pm 19.79$    \\
  &  & SVM &    $ \bf{0.8954 \pm 0.0096}$   & $57.63 \pm 20.56$        &    $\bf{33.77 \pm 21.06}$     &     $37.13 \pm 26.62$    \\
  &   & SVDD &     $0.8698 \pm 0.02129$   & $ 58.14 \pm 10.62$        &     $39.62 \pm 25.45$     &       $34.50 \pm 21.65$     \\
  &   &  CSSLM &    $0.8681 \pm 0.0273$   & $60.02 \pm 7.95$        &   $36.23 \pm 24.27$       &   $36.20\pm 19.35$           \\ \cline{2-7}    
   & \multirow{5}{*}{AD} & DT &    $ 0.8822 \pm 0.0016$  &     $2.73 \pm 8.64$  & $ 0.19 \pm 0.60$        &    $99.24 \pm 2.40$         \\ 
   & & KNN &    $ 0.8419 \pm 0.0146$    & $ 0.00 \pm 0.00$       &    $0.00 \pm 0.00$     &     $100.00 \pm 0.00$    \\ 
   & & SVM  &    $0.8802 \pm 0.0199$  & $ 48.78 \pm 26.46$         &    $47.55 \pm 32.97$      &     $30.41 \pm 22.60$         \\
   & & SVDD &     $0.8646 \pm 0.02512$ &     $49.18 \pm 21.02$         &     $46.04 \pm 35.42$     &       $ 32.98 \pm 26.27 $      \\
   & & CSSLM &    $\bf{0.8665 \pm 0.0267}$  &      $\bf{62.04 \pm 5.72}$     &   $\bf{28.87 \pm 16.50}$       &   $\bf{42.87 \pm 15.30}$        \\  \hline
\multirow{10}{*}{BCW}&
    \multirow{5}{*}{BC} & DT &    $0.9215 \pm 0.0452$  &     $87.27 \pm 1.58$         &    $11.43 \pm 4.22$     &     $13.83 \pm 4.64$    \\ 
  &  & KNN &    $ 0.9755 \pm 0.0043$     &     $88.75 \pm 1.52$       &    $13.33 \pm 3.53$     &     $9.07 \pm 1.40$    \\
  &  & SVM &    $  \bf{0.9882 \pm 0.0025}$   &     $\bf{91.62 \pm 1.59}$        &    $\bf{10.64 \pm 5.19}$     &     $5.89 \pm 3.18$    \\
   &  & SVDD &     $0.9607 \pm 0.0091$    &     $85.69 \pm 1.45$       &     $18.73 \pm 2.22$     &       $9.63 \pm 1.40$     \\
   &  &  CSSLM &    $0.9850 \pm 0.0022$    &     $83.90 \pm 3.14$        &   $29.05 \pm 5.35$       &   $\bf{0.65 \pm 1.08}$           \\ \cline{2-7}
   & \multirow{5}{*}{AD} & DT &    $ 0.8819 \pm 0.0749$  &     $78.81 \pm 7.01$         &    $\bf{8.10 \pm 2.64}$     &     $31.68 \pm 12.86$    \\ 
   & & KNN &    $ 0.9544 \pm 0.0196$        &     $44.89 \pm 38.89$    &    $1.91 \pm 2.57$     &     $65.05 \pm 31.04$    \\ 
   & & SVM  &    $0.9658 \pm 0.0102$   &     $41.90 \pm 44.24$            &    $2.54 \pm 3.01$      &     $62.90 \pm 39.45$         \\
   & & SVDD &     $0.9466 \pm 0.0222$   &     $76.07 \pm 26.86$       &     $13.49 \pm 5.95$     &       $ 24.21 \pm 27.18 $      \\
   & & CSSLM &    $\bf{0.9678 \pm 0.0186}$   &     $\bf{81.91 \pm 3.25}$         &   $30.79 \pm 4.62$       &   $\bf{2.99 \pm 1.96} $        \\
    \hline
   \multirow{10}{*}{KDD-Cup'99} &
    \multirow{5}{*}{BC} & DT &    $ 0.9893 \pm 0.0013$   & $94.50 \pm 0.45$       &    $0.92 \pm 0.98$     &     $9.86 \pm 0.25$    \\ 
    & & KNN &    $ 0.9909 \pm 0.0002$    &  $\bf{94.81 \pm 0.11}$      &    $\bf{0.33 \pm 0.13}$     &     $9.81 \pm 0.14$    \\
    & & SVM &    $ 0.9911 \pm 0.0060$   &  $90.64 \pm 3.63$        &    $0.80 \pm 0.65$     &     $17.06 \pm 6.67$    \\
    & & SVDD &     $0.9952 \pm 0.0013$  &  $86.22 \pm 6.20$        &     $21.98 \pm 14.87$     &       $\bf{3.57 \pm 4.14}$     \\
    & &  CSSLM &    $\bf{0.9954 \pm 0.0014}$  &  $87.40 \pm 3.10$         &   $17.22 \pm 8.70$       &   $3.78 \pm 1.78$           \\ \cline{2-7}
    &\multirow{5}{*}{AD} & DT &    $ 0.9883 \pm 0.0025$   &  $\bf{93.80 \pm 1.43}$          &    $0.18 \pm 0.43$     &     $11.84 \pm 2.64$    \\ 
   & & KNN &    $ 0.9767 \pm 0.0008$     &  $87.07 \pm 0.32$       &    $0.83 \pm 0.48$     &     $23.55 \pm 0.47$    \\ 
   & & SVM  &    $0.9971 \pm 0.0006$   &  $81.19 \pm 8.98$         &    $\bf{0.01 \pm 0.03}$      &     $33.34 \pm 14.80$         \\
   & & SVDD &     $0.9948 \pm 0.0015$  &       $ 85.67 \pm 6.40 $  &  $23.04 \pm 15.08$        &     $3.45 \pm 4.08$           \\
   & & CSSLM &    $\bf{0.9967 \pm 0.0019}$     &  $87.81 \pm 5.48$        &   $20.32 \pm 11.75$       &   $\bf{2.56 \pm 2.99} $        \\
    \hline \hline
    \end{tabular}	}
\end{table}

The MC dataset is a 2D synthetic dataset where positive points are sampled from a multivariate Gaussian distribution and negative points from a multivariate Cauchy distribution (see Figure \ref{fig:Cauchy_data_distribuation}). 
The dataset's heavy-tailed, non-Gaussian characteristics provide a challenging test case for evaluating robustness against outliers and heavy-tailed noise in statistical learning. For MC dataset  the results in Table \ref{table:resultofCPEAD} show that in BC, CSSLM achieves the highest AUPRC ($0.9944 \pm 0.0002$), slightly higher than SVM ($0.9942 \pm 0.0006$), while attaining the lowest negative-class error ($e_- = 3.60\% \pm 0.49$). In AD, CSSLM secures the best G-means ($87.61\% \pm 8.60 $) together with the lowest negative-class error ($e_- = 5.59 \% \pm 2.97$). These results highlight the advantage of CSSLM in controlling negative-class errors while remaining competitive in AUPRC and G-means.
Key observations: (i) CSSLM dominates in $e_-$ across both scenarios; (ii) improvements are consistent across runs with relatively low variance.

\begin{figure}[h]
	\centering
	\includegraphics[width=0.5\linewidth]{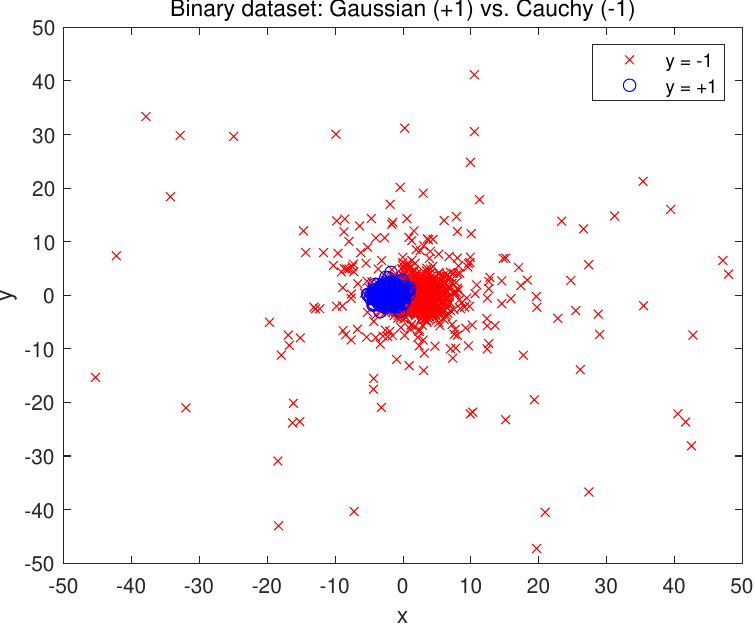}
	\caption{The data distribution of the Multivariate Cauchy Dataset.}
	\label{fig:Cauchy_data_distribuation}
\end{figure}

The BTSC dataset\footnote{\href{https://archive.ics.uci.edu/dataset/176/blood+transfusion+service+center}{https://archive.ics.uci.edu/dataset/176/blood+transfusion+service+center}} contains 4-dimensional donor data from Hsin-Chu City, Taiwan. This real-world data set captures the demographic and behavioral patterns of blood donors, making it particularly valuable for predicting donor return behavior and evaluating classification methods. Before modeling, we apply logarithmic transformation to the features. On this more challenging dataset, many methods degrade significantly. In BC, CSSLM maintains stable performance (AUPRC $= 0.8681 \pm 0.0273$, G-means $= 60.02\% \pm 7.95$), whereas KNN collapses with G-means $= 6.61\% \pm 20.90$. In AD, CSSLM again sustains performance (AUPRC $= 0.8665 \pm 0.0267$, G-means $= 62.04\% \pm 5.72$), while other methods all break down.
Key observations: (i) CSSLM is the only method that avoids collapse in both scenarios; (ii) it consistently yields relatively lower $e_-$.

The BCW dataset\footnote{\href{https://archive.ics.uci.edu/dataset/17/breast+cancer+wisconsin+diagnostic}{https://archive.ics.uci.edu/dataset/17/breast+cancer+wisconsin+diagnostic}} contains 30 features derived from digitized images of fine needle aspirates of breast masses, describing characteristics of cell nuclei. This medical dataset is widely used to benchmark classification methods in cancer diagnosis. Prior to modeling, we apply log transformation to the data. 
In BC, CSSLM performs competitively (AUPRC $= 0.9850 \pm 0.0022$), close to SVM ($0.9882 \pm 0.0025$), while achieving the lowest negative-class error ($ 0.65\% \pm 1.08$). In AD, CSSLM achieves the best AUPRC ($0.9678 \pm 0.0186$), best G-means  ($81.91 \pm 3.25$) and again the lowest negative-class error ($ 2.99 \% \pm 1.96$). Meanwhile, both KNN and SVM suffer breakdown under AD.
Key observations: (i) CSSLM balances $e_+$ and $e_-$ more reliably than baselines; (ii) KNN minimizes $e_+$ but incurs substantially higher $e_-$.

The KDD-Cup'99 dataset\footnote{\href{https://kdd.ics.uci.edu/databases/kddcup99/kddcup99.html}{https://kdd.ics.uci.edu/databases/kddcup99/kddcup99.html}} is a 41-dimensional benchmark for network intrusion detection, containing simulated traffic data with normal connections and various attacks. We use the 10\% training subset ``\texttt{kddcup.data\_10\_percent}" and the full test set ``\texttt{kddcup.data\allowbreak.corrected}".   Categorical attributes are converted via one-hot encoding, followed by min-max normalization. 
In BC, CSSLM achieves the best AUPRC ($0.9954 \pm 0.0014$) and a slightly higher negative-class error ($ 3.78\% \pm 1.78$) compared to SVDD ($3.57\%\pm 4.14$). In AD, CSSLM again dominates with the highest AUPRC ($0.9967 \pm 0.0019$) and lowest $e_-$ ($2.56 \% \pm 2.99$), whereas SVM attains a low $e_+$ but at the cost of very high $e_-$. The key observation is CSSLM consistently secures both the best AUPRC and relatively low $e_-$. 

For comparison with deep models we additionally train MLP and CNN on KDD-Cup’99  with {\em all} the training set ($(m,n)=(97278,396743)$) of the KDD-Cup’99 Dataset. Here we adopt a 6-layer MLP and a 1-layer CNN following \cite{vinayakumar2017applying}, trained with learning rate $0.1$, batch size $128$, and $50$ epochs. 
Table \ref{table:kdd99 result for DD} shows the test results. Even when MLP (AUPRC $= 0.9973$) and CNN ($0.9940$) are trained on the {\it full} training set, their performance remains comparable to CSSLM trained on much smaller samples. This underscores the efficiency of support vector-based methods, whose performance depends primarily on the number of support vectors rather than the training set size. Hence the support vector-based approaches remain highly sample-efficient compared to deep neural models.

\begin{table}[h!]
    \centering
   \caption{\small The AUPRC, G-means, test errors ($e_+$, $e_-$) for   MLP and CNN on the KDD-Cup'99 Dataset.}
    \label{table:kdd99 result for DD}
        {\scriptsize 
    \begin{tabular}{ccccc}
    \hline
    \hline
     Model & AUPRC & G-means &$e_+(\%)$ & $e_-(\%)$ \\
    \hline    
     MLP & \bf{0.9973} & \bf{96.58} & 1.33 & \bf{5.47}\\
    CNN & 0.9940 & 96.48 & \bf{1.01} &  5.96\\
    \hline
    \hline
    \end{tabular}	}
\end{table}

\subsection{Decision Boundary Visualization on the Banana Dataset}

We examine the differences in the decision boundaries learned by the SVM-based models. To this end, we apply the Gaussian-kernel models (SVM, SVDD, and CSSLM) to Scenarios C and D of the banana dataset in \ref{tab.data}. The SVDD and CSSLM are solved via their dual formulations using \texttt{quadprog} in Matlab, while SVM is trained with \texttt{libsvm}. 
For comparison, we also include the performance of  DT and KNN are included and we adopted the built-in \texttt{fitctree} and \texttt{fitcknn} functions in Matlab solving DT and KNN, respectively.
Table \ref{table:result_cmp} reports the parameters chosen together with the test performance.  Figure \ref{fig:result_cmp} presents the AUPRC curves of all the methods in the test set.   Figure \ref{fig:boundary-training} illustrates the decision boundaries learned by SVM, SVDD, and CSSLM on the training data. Figure \ref{fig:boundary-test} further overlays these boundaries on the test data for direct comparison.

\begin{table}[h!]
\centering
\caption{\small Hyperparameters and evaluation metrics of each models on Scenarios C and D of the Banana Dataset. }
	\label{table:result_cmp}
 	    {\scriptsize 
		\begin{tabular}{clcccccccc}
			\hline\hline
            \multirow{2}{*}{Scenario} & \multirow{2}{*}{Model} & \multicolumn{4}{c}{Hyperparameter} & \multicolumn{4}{c}{Evaluation metric} \\
            \cline{3-10}
			& & $\mu$           & $\nu$   & $b$        & $\gamma$           & AUPRC  & G-means $(\%)$             & $e_+(\%)$    & $e_-(\%)$ \\
			\hline\hline
			\multirow{5}{*}{C} &   DT    & -            & -     & -     & -      & $0.9642$  & $90.94$  & $12.34$      & $5.65$\\   
            &   KNN    & -            & -     & -     & -      & $0.9908$  & $92.52$   & $10.50$      & $\bf{4.36}$\\   
            &   SVM    & -            & $5.0000$     & $0.6234$     & $0.3762$   & $0.9920$ & $\bf{93.08}$    & $6.56$      & $7.27$\\
                &   SVDD    & $0$            & $0.0038$     & $0.3117$     & $0.1881$      & $0.9871$  & $92.63$   & $\bf{4.97}$      & $9.69$\\
                &   CSSLM    & $0.0384$            & $0.0038$     & $0.1558$     & $0.3762$      & $\bf{0.9925}$  & $92.91$   & $8.66$      & $5.49$\\
            \hline
			\multirow{5}{*}{D}  &   DT    & -            & -     & -     & -      & $0.9535$ & $83.62$    & $0.26$      & $29.87$\\   
            &   KNN    & -            & -     & -     & -      & $0.9889$  & $78.25$   & $\bf{0.00}$      & $38.77$\\  
            &     SVM    & -                & $0.5000$     & $5.0000$     & $0.4385$      & $0.9920$  & $91.28$   & $1.58$     & $15.35$\\
             &     SVDD    & $0$                & $0.0091$     & $2.5000$     & $0.2193$      & $0.9887$ & $92.30$   & $2.89$     & $12.28$\\
             & CSSLM & $0.0909$                & $ 0.0091$     & $2.5000$     & $0.4385$      & $\bf{0.9921}$  & $\bf{92.42}$   & $10.24$     & $\bf{4.85}$\\
			\hline\hline
		\end{tabular}}
 \end{table}

 \begin{figure}[h!]
\centering
\subfigure[Scenario C]{\label{fig:AUPRC curve in Scenario C}\includegraphics[width=0.45\textwidth]{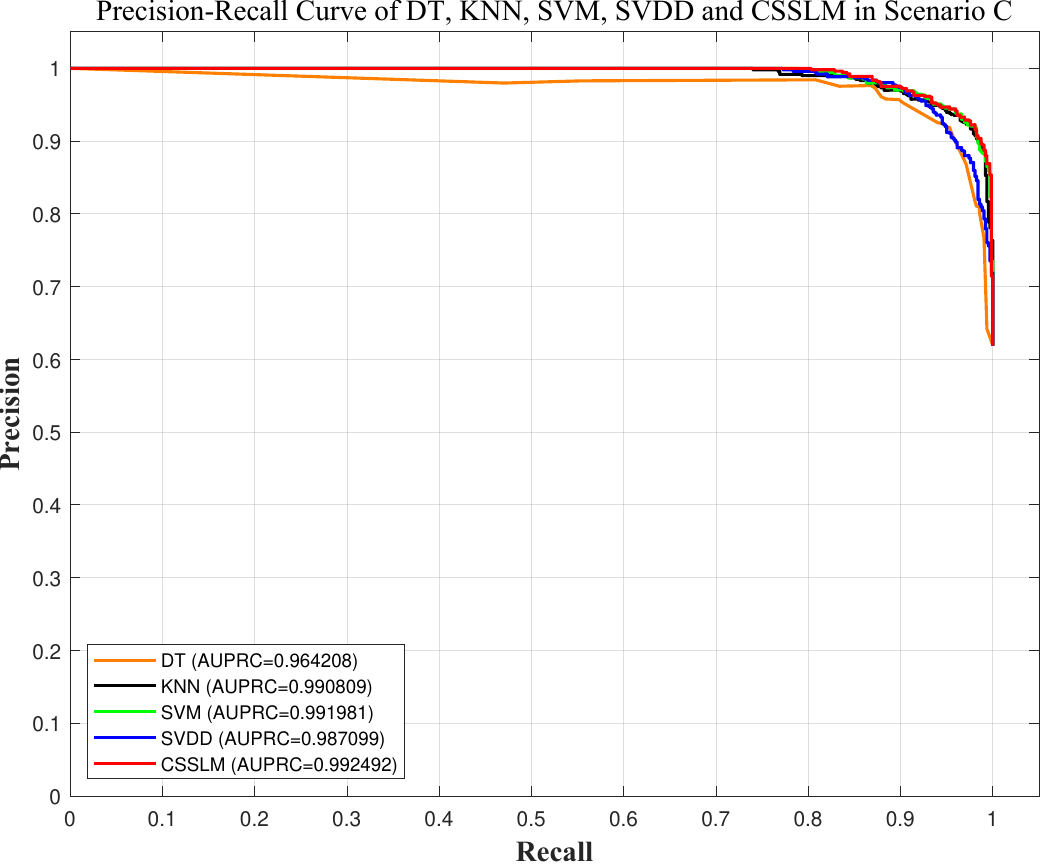}}
\subfigure[Scenario D]{\label{fig:AUPRC curve in Scenario D}\includegraphics[width=0.45\textwidth]{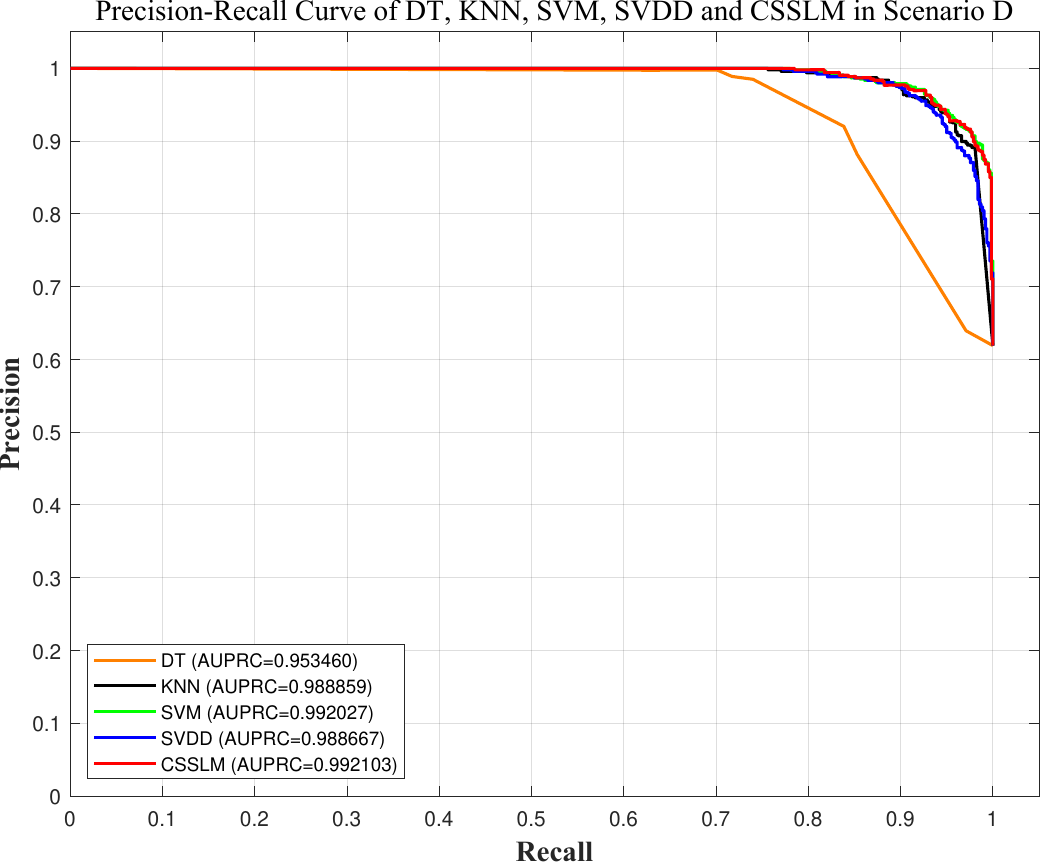}}
\caption{\small Comparison of precision–recall curves for DT, KNN, SVM, SVDD, and CSSLM in Scenarios C and D on the Banana Dataset.}
\label{fig:result_cmp}
\end{figure}

 \begin{figure}[h!]
	\centering
	\subfigure[Scenario C]{
		\begin{minipage}[t]{0.45\linewidth}
			\centering
			\includegraphics[width=\linewidth]{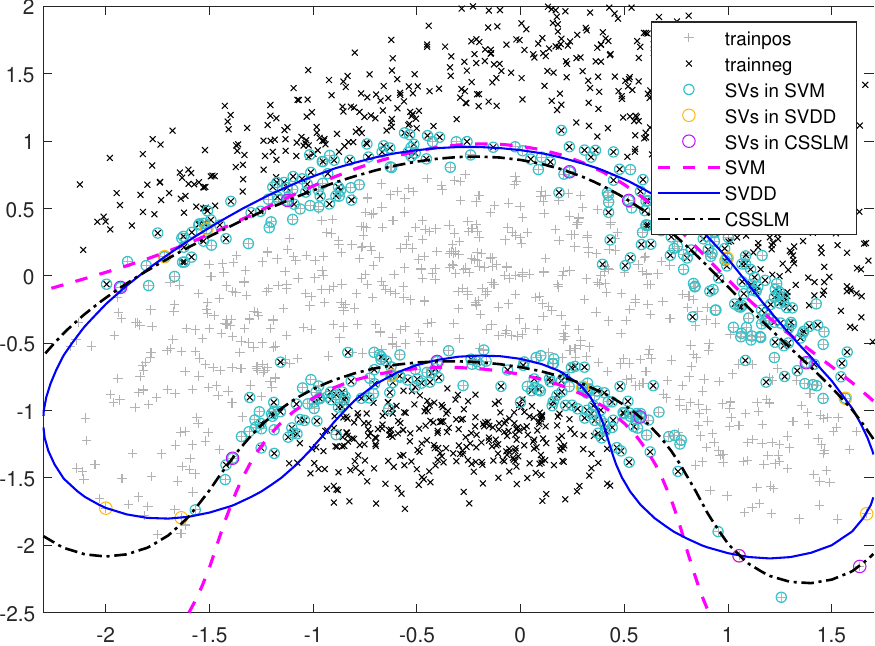}
			\label{fig:ScenarioA}			
		\end{minipage}
	}
    \hfill
	\subfigure[Scenario D]{
		\begin{minipage}[t]{0.45\linewidth}
			\centering
			\includegraphics[width=\linewidth]{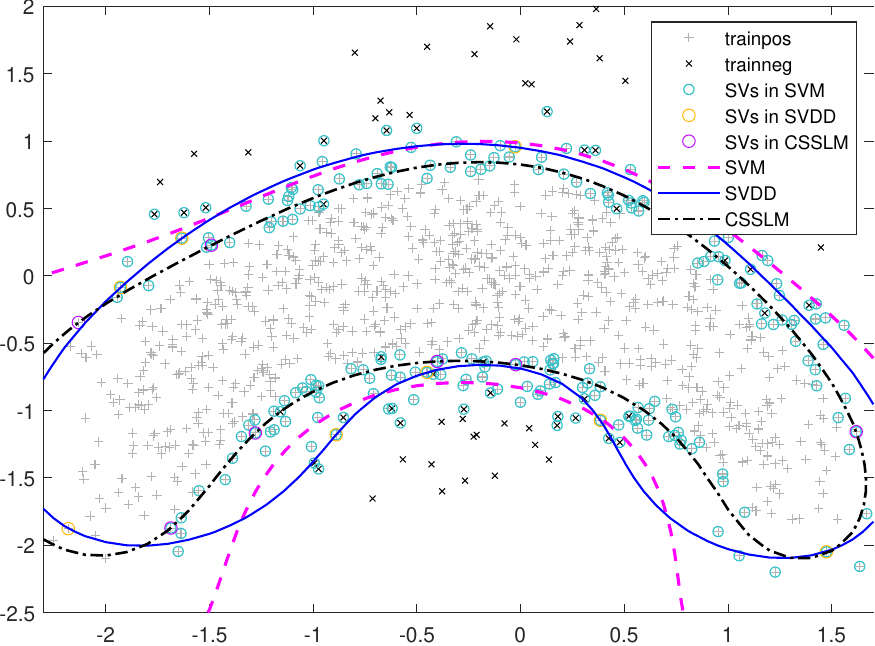}
			\label{fig:ScenarioB}		
		\end{minipage}
	}
\caption{\small Decision boundaries learned by SVM, SVDD, and CSSLM on the training data in Scenarios C and D on the Banana Dataset. Data points inside the circles indicate the supporting vectors.}
	 \label{fig:boundary-training}
\end{figure}

\begin{figure}[h!]
	\centering
	\subfigure[SVM ]{
		\begin{minipage}[t]{0.3\linewidth}
			\centering
			\includegraphics[width=\linewidth]{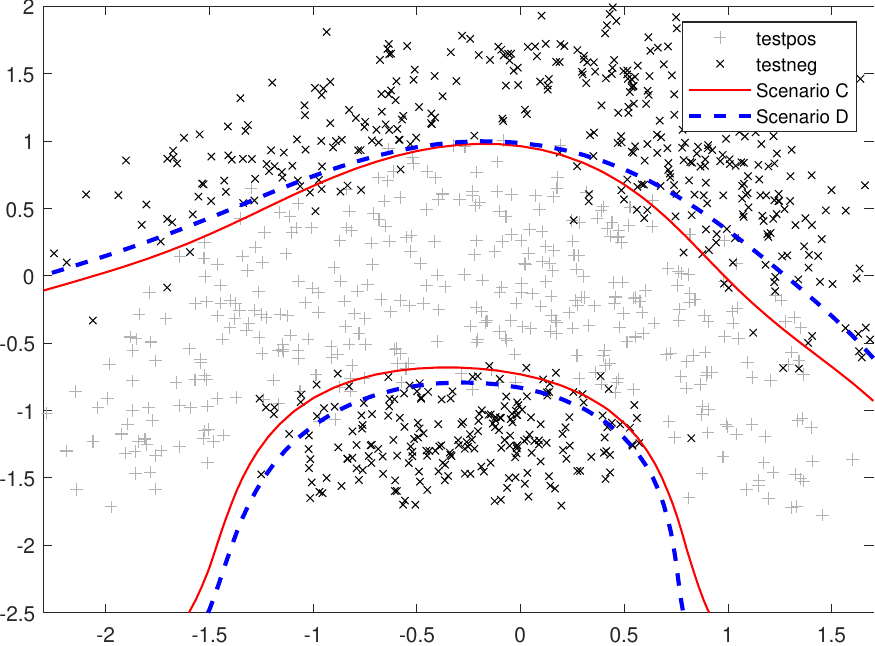}
			\label{fig:SVM}					
		\end{minipage}
		}
	\subfigure[SVDD]{
		\begin{minipage}[t]{0.3\linewidth}
			\centering
			\includegraphics[width=\linewidth]{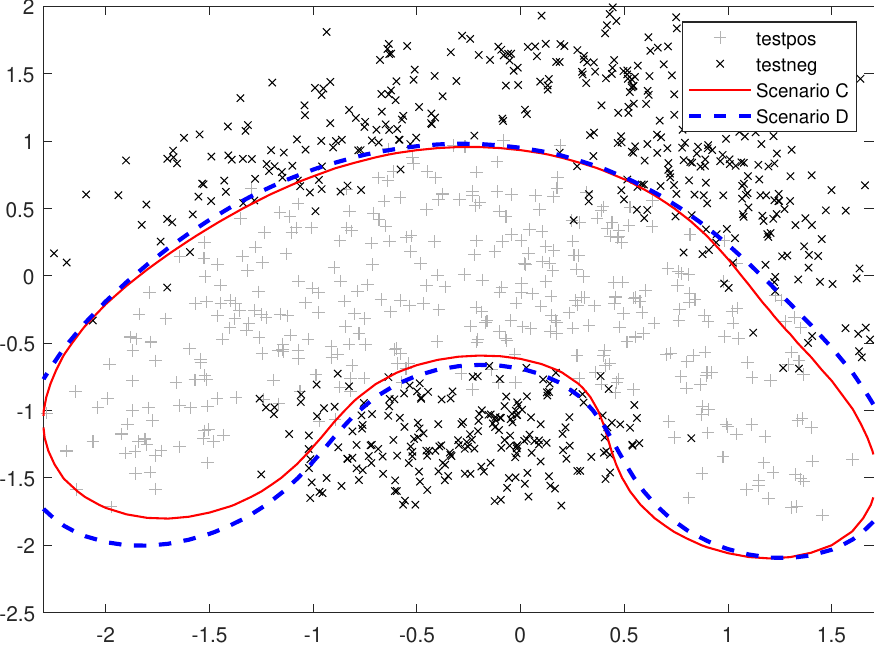}
			\label{fig:SVDD}		
		\end{minipage}
	}	
    \subfigure[CSSLM]{
		\begin{minipage}[t]{0.3\linewidth}
			\centering
			\includegraphics[width=\linewidth]{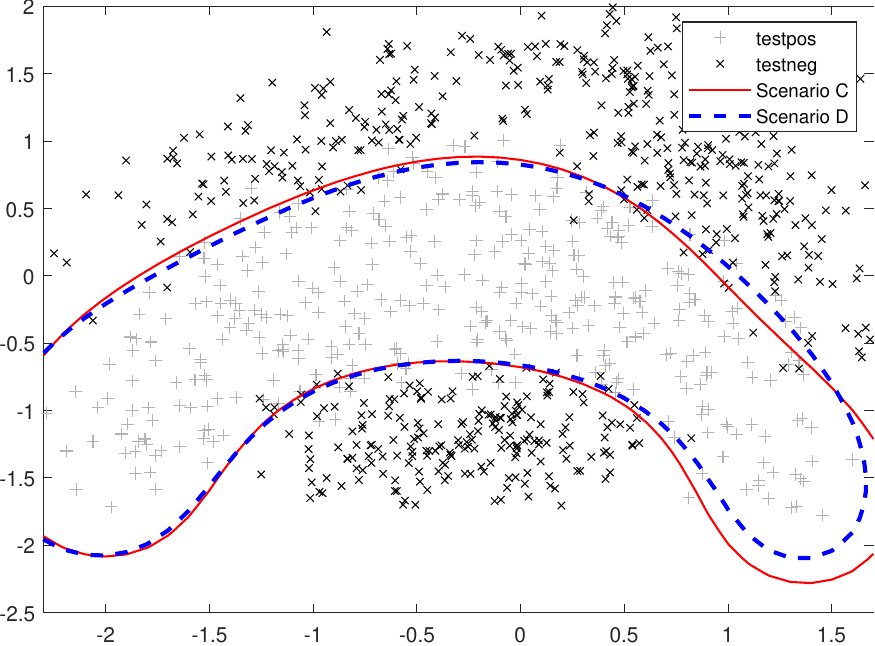}
			\label{fig:CSSLM}		
		\end{minipage}
	}	
\caption{\small Decision boundaries learned by SVM, SVDD, and CSSLM on the test data in Scenarios C and D on the Banana Dataset.}
	\label{fig:boundary-test}
\end{figure}

 In Scenario C, CSSLM achieves a false-positive rate $ 5.49\%$, which is significantly lower than SVM (7.27\%) and SVDD (9.69\%), at the cost of a slightly higher false-negative rate $ 8.66\%$ compared to SVM (6.56\%) and SVDD (4.97\%). This indicates that CSSLM adopts a conservative strategy to reduce false positives, which explains its slightly higher overall test error. In Scenario D, where negative samples are scarce, CSSLM attains a low false-negative rate $10.24 \%$  along with a very competitive false-positive rate $4.85\%$, demonstrating a more balanced performance compared to SVM ($e_+=1.58\%$, $e_-=
15.35\%$).

 Under these metrics, CSSLM consistently achieves the best or near-best performance, with the highest AUPRC in Scenario C ($0.9925$) and Scenario D ($0.9921$). Across both scenarios, CSSLM consistently achieves lower negative-class errors. Figure \ref{fig:boundary-test}  illustrates SVDD and CSSLM consistently identify similar banana-shaped boundaries across datasets, whereas SVM with a Gaussian kernel tends to overfit positive samples, explaining its lower positive-class error but higher negative-class error. Overall, these results suggest that although CSSLM may not always minimize raw test error in heavily imbalanced settings, it provides more balanced error trade-offs and superior robustness across scenarios. 

\subsection{ Solving Primal CSSLM via Stochastic Subgradient Method} 

We address the computational challenge of solving the proposed CSSLM in this subsection. The dual CSSLM \eqref{prob.dual1}  involves one linear equality and inequality constrains with box constraints. It is more simple that the primal CSSLM \eqref{prob.primaldual1}. This is a common phenomenon in SVM-like models. Similar to SVM \citep{libsvm11}, various coordinate descent methods and stochastic coordinate descent methods can be designed to solve the dual CSSLM \eqref{prob.dual1}. However it's not easy to describe them clearly in a limited space. 

Now we turn to solve the primal CSSLM \eqref{prob.primaldual1} which is a composition optimization, the sum of the smooth regularization term and the nonsmooth empirical loss function.  Motivated by the well-known Pegasos \citep{Pegasos11}, a stochastic sub-gradient method for SVM, we propose a stochastic subgradient method (SSGM) to solve \eqref{prob.primaldual1}. This approach is especially advantageous in large-scale settings.  
The SSGM operates with a batch size of 1 and adopts a decreasing learning rate $\eta_k = 1/(\nu \cdot k)$ at iteration $k$.  A data point with index $i_k \in [1,\ell]$ is selected. The   subgradient of the algorithm at $\bm w_k = \left(\bm a_k,s_k,t_k\right)$ is expressed as
\[
\nabla_{\bm w_k} := \left( \nu \bm a_k, -\frac{\nu}{2}, -\frac{\mu}{2} \right) + \nabla \phi_{i_k}(\bm w_k),
\]
where the example-dependent component $\nabla \phi_{i_k}(\bm w_k)$ is defined as

\[
\nabla \phi_{i_k}(\bm w_k) = 
\begin{cases}
\left(-\Phi(\bm x_{i_k}), \frac{1}{2}, 0\right) \cdot \mathds{1}\left[\langle \bm a_k, \Phi(\bm x_{i_k})\rangle < \frac{\|\Phi(\bm x_{i_k})\|^2 + s_k}{2}\right], & y_{i_k} = 1, \\
b\left(\Phi(\bm x_{i_k}), -\frac{1}{2}, \frac{1}{2}\right) \cdot \mathds{1} \left[\langle \bm a_k, \Phi(\bm x_{i_k})\rangle > \frac{\|\Phi(\bm x_{i_k})\|^2 + s_k - t_k}{2}\right], & y_{i_k} = -1,
\end{cases}
\]
with $ \bm a_k= \sum_{j=1}^{\ell} \beta_j^k \Phi(\bm x_j)$ and $\mathds{1}$ representing the indicator function.

\begin{figure}[h]
	\centering
	\includegraphics[width=0.55\linewidth]{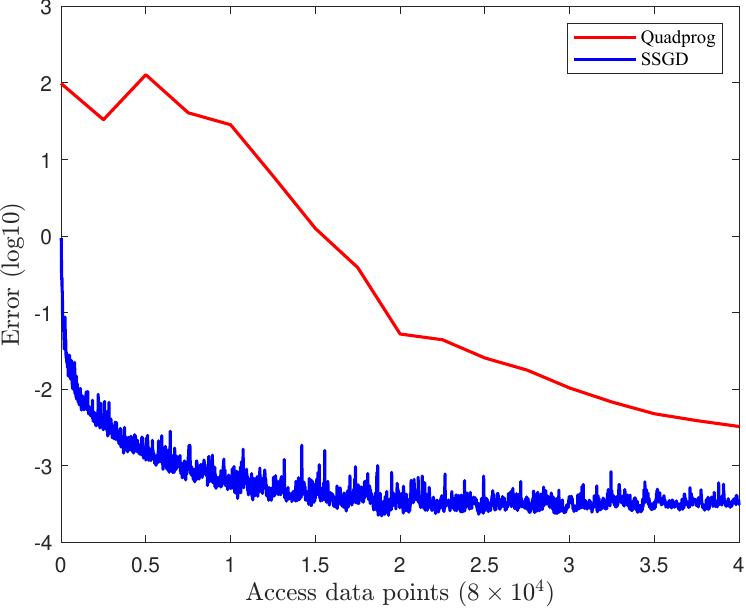}
	\caption{\small Comparison of optimization error between SSGM for solving CSSLM  \eqref{prob.primaldual1} and quadprog for solving the dual CSSLM  \eqref{prob.dual1}. }
	\label{fig:SSGD and Quadprog}
\end{figure}

We compare SSGM against the \texttt{quadprog} solver in MATLAB on the KDD-Cup'99 Dataset with 20k samples ($m=10000, n=10000$). The parameters in \eqref{prob.primaldual1} were set as $b=1, \nu=0.001$ and $\mu=0.05$. Again the gaussian kenel is used with $\gamma=1$. Figure \ref{fig:SSGD and Quadprog} illustrates how the relative error changes with the number of data points accessed where the optimal value means the result obtained by solving \eqref{prob.dual1} using \texttt{quadprog}.
As shown in Figure \ref{fig:SSGD and Quadprog}, SSGM significantly reduces computational cost compared to \texttt{quadprog}, especially as the dataset size increases. While \texttt{quadprog} computes an exact solution at high computational expense, SSGM achieves a rapidly decaying optimization error with far fewer data accesses. This demonstrates that SSGM is a scalable and practical alternative for large-scale applications of CSSLM.

\section{Conclusion}\label{sec.conclude}

We proposed a novel convex SSLM formulation that, for specific hyperparameter choices, reduces to a convex quadratic programming problem. The convexity of our model enables rigorous analysis and derivation of results that are challenging or impossible for traditional nonconvex SSLM models, notably simplifying the study of hypersphere-based SVM solutions. We systematically investigated the influence of hyperparameters on the optimal solution and established explicit connections with traditional SSLM, identifying conditions under which the solution is unique. Moreover, we derived the $\nu$-property, providing insight into how hyperparameters govern the distribution of support vectors and margin errors across positive and negative classes.  

Extensive numerical experiments demonstrate that our method consistently outperforms both the primal SSLM and its dual in terms of efficiency and accuracy. By addressing the challenges of nonconvexity, our work paves the way for developing more robust algorithms for hypersphere-based SVMs. Future work will explore additional theoretical properties and algorithmic strategies that leverage the convexity of our formulation.

\vspace{2cm}
{\noindent \bf Acknowledgment}
This work was initiated through discussions between the first author and Prof. Shuisheng Zhou at Xidian University on the role of nonconvexity in anomaly detection. The authors gratefully acknowledge Prof. Zhou for his many insightful suggestions and constructive comments during the preparation of this paper. The importance of convex modeling in the context of bilevel optimization was inspired by discussions with Dr. Samuel Ward from the University of Southampton. The authors sincerely thank Dr. Ward for highlighting this point and for his valuable feedback.

Hongying Liu is supported by the National Natural Science Foundation of China under Grant  12171021 and 12131004.

\bibliography{reference}
\begin{appendices}

\section{Proof of \textcolor{red}{Lemma} 1 in the ill-posed cases}\label{sec:appendA}

\begin{proof}
If  $\mu>m/{\ell}$,  setting $(\ba,r,t) = ( \alpha \tilde \ba, 0, \alpha^2\|\tilde \ba\|^2)$ with any $\alpha>0$ and   $\tilde \ba\neq \bm 0$  yields
\[
\begin{array}{rl}
&g( \alpha \tilde \ba, 0, \alpha^2\|\tilde \ba\|^2)\\
= & - \frac{1}{2} \mu \alpha^2\|\tilde \ba\|^2 +\frac{1}{2\ell}\sum\limits_{i=1}^m \|\Phi(\bx_i)-\alpha \tilde \ba \|^2  +\frac{b}{2\ell}\sum\limits_{i=m+1}^\ell\left(\alpha^2\|\tilde \ba\|^2 - \|\Phi(\bx_i)-\alpha \tilde \ba\|^2\right)_+\\
= &\frac{1}{2}(\tfrac{m}{\ell}-\mu)\alpha^2 \|\tilde\ba\|^2-\tfrac{1}{\ell}\sum\limits_{i=1}^m\langle\Phi(\bx_i),\tilde\ba\rangle \alpha + \frac{b}{\ell}\sum\limits_{i=m+1}^{\ell}\left(\langle \Phi(\bx_i),\tilde\ba\rangle \alpha -\frac{k_{ii}}{2}\right)_+ + \tfrac{1}{2\ell}\sum\limits_{i=1}^mk_{ii}\\
\le& \frac{1}{2}(\tfrac{m}{\ell}-\mu)\alpha^2 \|\tilde\ba\|^2- \tfrac{1}{\ell}\left[\sum\limits_{i=1}^m\langle\Phi(\bx_i),\tilde\ba\rangle  - b\sum\limits_{i=m+1}^{\ell}\left( \langle \Phi(\bx_i),\tilde\ba\rangle\right)_+\right]\alpha +\tfrac{1}{2\ell}\sum\limits_{i=1}^mk_{ii},
\end{array}
\]
where the inequality is by $k_{ii}\ge 0$.  It follows that $g( \alpha \tilde \ba, 0, \alpha^2\|\tilde \ba\|^2)\to -\infty$ as $\alpha\to +\infty$, since $ \mu > m/{\ell}$.

If $\mu>bn/{\ell}$,  setting $(\ba,r,t) = ( 0, 0, t )$ with sufficiently large $t>0$ yields
\[
g(0,0,t)=\tfrac{1}{2}\Big(\tfrac{bn}{\ell}-\mu\Big)t+\tfrac{1}{2\ell }\sum_{i=1}^m\|\Phi(\bx_i)\|^2-\tfrac{b}{2\ell}\sum_{i=m+1}^{\ell}\|\Phi(\bx_i)\|^2,
\]
which implies  $g(0,0,t)\to -\infty$ as $t\to +\infty$.
\end{proof}

\end{appendices}

\begin{appendices}

\section{Proof of Theorem \ref{prop:SSLM-De}}\label{sec:appendB}

 \begin{proof}
 (i) Substituting $\mu=0$ into $g(a,r,t)$, one has
 \[
 \begin{array}{ll}
 g(\ba,r,t)&=\frac{\nu}{2}r+\frac{1}{2\ell}\sum\limits_{i=1}^m(\|\Phi(\bx_i)-\ba\|^2-r)_+ +\frac{b}{2\ell}\sum\limits_{i=m+1}^\ell( r+t-\|\Phi(\bx_i)-\ba\|^2)_+\\
 &\ge \frac{1}{2}(\nu-\tfrac{m}{\ell})r+\frac{1}{2\ell }\sum\limits_{i=1}^m\|\Phi(\bx_i)-\ba\|^2\\
 & \ge \frac{1}{2\ell }\sum\limits_{i=1}^m\|\Phi(\bx_i)-  \ba\|^2\\
 & \ge \frac{1}{2\ell }\sum\limits_{i=1}^m\|\Phi(\bx_i)-\hat \ba\|^2.
 \end{array}
 \]
Here  the first inequality holds as  equality   if and only if $r \le \|\Phi(\bx_i)-\ba\|^2,   i\in [1,m]$ and $ r+t \le \|\Phi(\bx_i)-\ba\|^2,  i\in [m+1, \ell]$;
 the second inequality holds because of  $\mu = 0, \nu \ge  m/\ell$ and the inequality holds as equality
if and only if $\ba = \hat \ba$ and $(\nu-m/{\ell})r/2 = 0$. The last inequality is by the definition \eqref{eq:degenerate-spherecenter}  of $\hat \ba$.  Therefore,  $g(\ba, r, t)$ reaches its lower bound
if and only $(\ba, r, t) \in \Omega$.

(ii)
Suppose $\nu+\mu \ge m/{\ell}$ is satisfied. We first prove \eqref{prob.unc} is equivalent to
\begin{equation}\label{prob.unc.rt}
  \inf\limits_{\ba\in H, t\in\mathbb{R}} g(\ba, 0, t).\end{equation}
Moreover, if  $\inf\limits_{\ba\in H, t \in \mathbb{R}} g(\ba, 0, t) > -\infty$,  then the global minimizer satisfies  $r=0$ and $t \ge 0$.
By $\nu+\mu \ge m/{\ell}$, one has $0 < \mu \le m/{\ell}$. Hence for any $\ba\in F,  r\geq0$,
\begin{equation*}\label{eq:degenerate}
\begin{array}{rl}
  g(\ba, r, t)
\geq& \frac{1}{2}(\nu r-\mu t)+\frac{1}{2\ell}\sum_{i=1}^m(\|\Phi(\bx_i)-\ba\|^2-r)+\frac{b}{2\ell}\sum_{i=m+1}^\ell (r+t-\|\Phi(\bx_i)-\ba\|^2)_+\\
=&\frac{1}{2}(\nu-\frac{m}{\ell})r-\frac{\mu}{2}t+\frac{1}{2\ell}\sum_{i=1}^m\|\Phi(\bx_i)-\ba\|^2+\frac{b}{2\ell}\sum_{i=m+1}^\ell ( r+t-\|\Phi(\bx_i)-\ba\|^2)_+\\
\ge&-\frac{\mu}{2}(r+t)+\frac{1}{2\ell}\sum_{i=1}^m\|\Phi(\bx_i)-\ba\|^2+\frac{b}{2\ell}\sum_{i=m+1}^\ell(r+t-\|\Phi(\bx_i)-\ba\|^2)_+\\
=& g(\ba,0,r+t) \\
\ge &  \inf\limits_{\ba\in H, t } g(\ba, 0, t)
\end{array}
\end{equation*}
where the first inequality comes from $(z)_+\ge z$ and
the second inequality  is due to condition $\nu+\mu \ge m/{\ell}$.
On the other hand, for any $t < 0$,
\begin{equation}\label{eq.ga0t}
 g(\ba, 0, t)    = -\tfrac{\mu}{2}t +\tfrac{1}{ 2\ell}\sum\limits_{i=1}^m\|\Phi(\bx_i)-\ba \|^2  \ge   \tfrac{1}{2\ell}\sum\limits_{i=1}^m\|\Phi(\bx_i)-\ba \|^2
=   g(\ba, 0, 0)
\end{equation}
by $\mu\ge 0$.
It follows that  $ g(\ba, r, t) \ge \inf\limits_{\ba\in H, t\in\R } g(\ba, 0, t) \ge  \inf\limits_{\ba\in H, t\ge 0} g(\ba, 0, t)$. It implies
\[\inf\limits_{\ba\in H, t\in\R } g(\ba, 0, t) =  \inf\limits_{\ba\in H, t\ge 0} g(\ba, 0, t).\]
Moreover, \eqref{eq.ga0t} indicates that if $\inf\limits_{\ba\in H, t } g(\ba, 0, t)  > - \infty$,
$(\ba_*,   t_*) \in  \arg\min\limits_{\ba\in H, t\in\mathbb{R}} g(\ba, 0, t)$ also satisfies
$t^* \ge 0$.
The former result implies that problem \eqref{prob.unc} in this case is equivalent to the unconstrained problem
\eqref{prob.unc.rt}, which is equivalent to a \emph{convex} problem by setting $z = \|\ba\|^2 - t$ in $g$:
 \begin{equation}\label{prob.without.s}
\min\limits_{\ba\in H,  z\in\R} \tilde f(\ba, z),
 \end{equation}
 with (see Appendix \ref{sec:appendD} for deriving $\tilde f(\ba, z)$)
 \[
  \tilde f(\ba, z)  =   \tfrac{\ell\lambda}{2} \|\ba\|^2 + \tfrac{\ell\mu}{2}z - m\langle\hat \ba, \ba\rangle
    + b\sum_{i=m+1}^\ell \left( - \tfrac{z+ k_{ii}}{2} + \langle\ba, \Phi(\bx_i)\rangle\right)_+.
  \]
  If $\mu < m/\ell \le \nu+\mu$,   then \eqref{prob.without.s} reverts to the convex quadratic program problem
\eqref{prob.qp}.

(iii)  If $ \mu  =   m/\ell$,  then \eqref{prob.without.s} reverts to the linear programming problem
\eqref{prob.lp}.
\end{proof}
\end{appendices}

\begin{appendices}

\section{Derivation of the objective function in the degenerate case}\label{sec:appendD}

It holds that
\[\begin{aligned}
 g(\ba, 0,t)
  = &   -\tfrac{\mu}{2}t+\tfrac{1}{2 \ell }\sum\limits_{i=1}^{m} \|\ba-\Phi(\bx_i)\|^2 + \tfrac{b}{2 \ell}\sum_{i=m+1}^\ell (t-  \|\Phi(\bx_i) - \ba\|_2^2)_+ \\
 = &   -\tfrac{\mu}{2}t+\tfrac{1}{2\ell  }\sum\limits_{i=1}^{m}(\|\ba\|^2  -2 \langle \ba, \Phi(\bx_i)\rangle+k_{ii}) + \tfrac{b}{2\ell}\sum_{i=m+1}^\ell (t-\|\ba\|^2+2\langle\ba, \Phi(\bx_i)\rangle-k_{ii})_+ \\
 = & - \tfrac{\mu}{2}t +\tfrac{m}{2\ell}\|\ba\|^2 - \tfrac{m}{\ell}\langle \hat\ba, \ba\rangle + \tfrac{b}{2\ell}\sum_{i=m+1}^\ell (t-\|\ba\|^2+2\langle\ba, \Phi(\bx_i)\rangle-k_{ii})_++ \tfrac{1}{2\ell}\sum_{i=1}^mk_{ii}
  \end{aligned} \]
 where $\hat\ba = \frac{1}{m}\sum_{i=1}^m\Phi(\bx_i)$.
Letting $z = \|\ba\|^2 - t$ in $f$,  it follows that
\[
g(\ba,0,t)=\tfrac{1}{\ell}\tilde f(\ba,z)+ \tfrac{1}{2\ell}\sum_{i=1}^m k_{ii}
\]
where
\[
\begin{aligned}
\tilde f(\ba,z) =
& \tfrac{\ell}{2}\left(\tfrac{m}{\ell} -\mu\right) \|\ba\|^2 + \tfrac{\ell\mu}{2}z - m\langle\hat \ba, \ba\rangle
    + b\sum_{i=m+1}^\ell \left( - \tfrac{z+ k_{ii}}{2} + \langle\ba, \Phi(\bx_i)\rangle\right)_+.
  \end{aligned}
  \]
\end{appendices}

\end{document}